\documentclass[preprint,12pt]{elsarticle}

\usepackage{amssymb}
\usepackage{amsmath}
\usepackage{amsthm} 
\usepackage{mathrsfs}
\usepackage{mathtools}
\usepackage{bm}

\usepackage{xcolor}

\usepackage{booktabs}
\usepackage{natbib} 
\usepackage{comment}

\usepackage{graphicx}    
\usepackage{subcaption} 

\usepackage{algorithm}
\usepackage{algorithmic}

\usepackage{threeparttable}

\usepackage{makecell}
\usepackage{dsfont}
\usepackage{wrapfig}
\usepackage{subcaption}
\usepackage{changepage}

\newtheorem{theorem}{Theorem}
\newtheorem{proposition}{Proposition}
\newtheorem{lemma}{Lemma}

\theoremstyle{remark}
\newtheorem{remark}{Remark}
\newtheorem{assumption}{Assumption}

\newcommand{\gradw}{\nabla_{\omega}}

\journal{Journal of Computational Physics}

\begin{document}

\begin{frontmatter}

\title{Alternating Levenberg--Marquardt Training of Physics-Informed Neural Networks with Fourier-Enhanced Features}

\author{Yulun Wu, Matthieu Barreau, Miguel Aguiar,  \&  Karl H.~Johansson} %

\affiliation{organization={Division of Decision and Control Systems,
            School of Electrical Engineering and Computer Science,
            KTH Royal Institute of Technology},
            city={Stockholm},
            postcode={100 44},
            country={Sweden}}

\begin{abstract}
Physics-informed neural networks (PINNs) often fail to accurately resolve partial differential equations (PDEs) with high-frequency or multi-scale solutions, as well as strongly nonlinear problems. Two factors underlie this difficulty: spectral bias, the tendency of neural networks to underfit high-frequency features; and representation--coefficient coupling, the entanglement of representation learning and coefficient fitting within a single nonconvex optimization objective. In this work, we propose the Fourier-enhanced alternating Levenberg--Marquardt PINN (FALM-PINN), an optimization framework that decouples representation learning from coefficient fitting. The upper-level problem learns a Fourier-enhanced basis that enriches the latent space with high-frequency components, while the lower-level problem resolves the coupling by fitting the projection coefficients on this basis, solving a nonlinear least-squares problem with the Levenberg--Marquardt algorithm. The framework applies to general nonlinear and coupled PDE systems, and reduces to a single-step convex optimization problem for linear PDEs. We prove global convergence of the alternating training scheme in both cases. Numerical examples on multiple challenging high-frequency and nonlinear PDEs show that FALM-PINN achieves relative $L^2$ errors up to two orders of magnitude lower than state-of-the-art baselines.
\end{abstract}


\end{frontmatter}

\section{Introduction}
Partial differential equations (PDEs) are fundamental tools for physical modeling, describing phenomena ranging from wave propagation and fluid flow to quantum dynamics. Since analytical solutions are rarely available, classical methods are widely used to provide accurate approximations. These include finite difference~\citep{leveque2007finite} and finite element methods~\citep{zienkiewicz1977finite,ciarlet2002finite,brenner2008mathematical}, spectral methods~\citep{boyd2001chebyshev,canuto2006spectral,shen2011spectral}, multi-scale methods~\citep{weinan2003heterognous,abdulle2012heterogeneous,efendiev2009multiscale}, and specialized approaches for highly oscillatory problems~\citep{iserles2005efficient}. Although accurate and well understood, these methods often rely on carefully constructed meshes or basis functions, which can be difficult to design for complex geometries, and their cost grows rapidly in high dimensions. These limitations have motivated learning-based solvers, including physics-informed approaches that incorporate the governing equations directly into neural network training.

Physics-Informed Neural Networks (PINNs) have become one of the most widely used frameworks for solving PDEs \citep{raissi2019physics,karniadakis2021physics}. A standard PINN represents the solution using a multi-layer perceptron (MLP) and obtains its derivatives through automatic differentiation. This enables the encoding of the governing equation, along with boundary and initial conditions, into a composite loss. This mesh-free formulation applies to both forward and inverse problems and extends naturally to challenging settings such as high-dimensional PDEs~\citep{hu2024tackling}, complex geometries~\citep{costabal2024delta}, flow-field reconstruction from sparse or noisy data~\citep{raissi2020hidden,jin2021nsfnets}, and inverse design in nano-optics and metamaterials~\citep{chen2020physics}. 
More broadly, physics-informed learning has been integrated with Bayesian formulations for uncertainty quantification under noisy data~\citep{yang2021b}, and with diffusion models that impose PDE constraints during generative sampling~\citep{bastek2025physics,wang2025source,shu2023physics}. It has also been extended to infinite-dimensional, operator-learning settings, such as Fourier Neural Operator~\citep{li2021fourier} and DeepONet~\citep{lu2021learning}.

Despite this progress, several well-known challenges still limit the effectiveness of PINNs. A prominent issue is \textit{spectral bias}, the tendency of neural networks to learn low-frequency components first, which impedes the recovery of high-frequency and multi-scale solutions~\citep{rahaman2019spectral,xu2025understanding}. Training is also difficult for stiff, nonlinear, and chaotic problems, where standard PINNs may converge slowly or produce inaccurate solutions despite attaining a small training loss~\citep{krishnapriyan2021characterizing,wang2024respecting}. To address these challenges, a broad range of strategies has been developed to improve PINN trainability.

Loss-weighting methods adaptively balance the terms of the non-convex composite objective. For example, Wang et al.~\citep{wang2022and} analyze the neural tangent kernel to equalize the convergence rates of the physics and data loss terms. Pointwise weighting methods further adapt the contribution of individual points, loss-attentional PINNs employ an auxiliary network that adaptively emphasizes high-error points~\citep{song2024loss}; Anagnostopoulos et al.~\citep{anagnostopoulos2024residual} propose a gradient-free, residual-based attention scheme that updates pointwise weights directly from residual magnitudes; Si and Yan~\citep{si2026convolution} extend pointwise weighting to local neighborhoods from a primal-dual optimization perspective; and Zhao et al.~\citep{zhao2026casual} impose temporal causality through non-negative, monotonically decreasing pointwise weights that are decoupled from the collocation-point arrangement. A complementary direction improves the sampling of collocation points. Wu et al.~\citep{wu2023comprehensive} systematically compare non-adaptive and residual-based adaptive sampling strategies; Gao and Wang~\citep{gao2023active} select informative points by active learning for high-dimensional nonlinear PDEs; and Lau et al.~\citep{lau2024pinnacle} jointly select collocation and experimental points.

Beyond reweighting and resampling, another line of work enriches the network representation itself through input embeddings, activation functions, and architectural modifications. Fourier feature embeddings enrich the input representation with high-frequency components~\citep{tancik2020fourier}, broadening the frequency spectrum accessible to the network therefore potentially alleviating spectral bias~\citep{wang2021eigenvector}. At the level of activation functions, Sitzmann et al.~\citep{sitzmann2020implicit} adopt sinusoidal activations to represent oscillatory structures; Si et al.~\citep{si2025complexphysicsinformedneuralnetwork} introduce learnable Cauchy activations to enhance the approximation of stiff and high-frequency solutions; and Zeng and Zhu~\citep{zeng2026nurbs} introduce a sine-enhanced adaptive activation for solving PDEs on irregular domains. At the architectural level, Zhao et al.~\citep{ZhaoEtAl24} introduce transformer backbones that capture temporal dependencies in evolution equations; Wang et al.~\citep{wang2025kolmogorov} replace conventional MLP layers with Kolmogorov--Arnold layers built on learnable spline functions; Jagtap and Karniadakis~\citep{jagtap2020extended} decompose the space-time domain and assign a separate subnetwork to each subdomain; and Wang et al.~\citep{wang2024piratenets} stabilize the training of deep PINNs through adaptive residual connections combined with Fourier feature embeddings.

However, accurately resolving highly oscillatory or strongly multi-scale solutions remains challenging. Fourier feature and activation-based methods substantially alleviate spectral bias, but their accuracy may still plateau when the target solution spans a very wide frequency band or contains sharp structures, where even carefully tuned PINNs can struggle to achieve high precision~\citep{mustajab2024physics,wang2024multi}. More fundamentally, these advanced PINN variants train all network parameters jointly by minimizing a single nonconvex objective, in which feature learning and coefficient fitting remain entangled. From an adaptive-basis viewpoint, the effective basis changes throughout training while the output layer coefficients are simultaneously adjusted~\citep{cyr2020robust}. This coupling can induce slow convergence, reduce robustness, and complicate theoretical analysis.

Several methods mitigate this coupling by separating basis construction from coefficient fitting. Extreme learning machines~\citep{huang2006extreme} and their physics-informed variants~\citep{dwivedi2020physics,dong2021local} fix a randomly sampled feature basis and recover the output weights by linear least squares, which makes the coefficient problem convex and efficient to solve. However, the fixed basis limits approximation accuracy on complex or high-frequency solutions. Closest to our work, the iterative PINN with Fourier-enhanced features (IFeF-PINN)~\citep{wu2025iterative} makes the basis adaptive by applying a Fourier feature mapping to the last hidden layer, and is trained iteratively by updating the basis while solving for the output coefficients. IFeF-PINN is developed for scalar, linear PDEs, for which the coefficient subproblem is convex and admits a closed-form solution; extending the decoupled framework to nonlinear PDE systems, where neither the closed-form update nor the convexity-based convergence analysis carries over, remains open.

To bridge this gap, we propose the Fourier-enhanced alternating Levenberg--Marquardt PINN (FALM-PINN), a training framework that decouples adaptive basis learning from coefficient fitting and applies to general nonlinear and coupled PDE systems. Training alternates between an upper-level problem that updates the hidden layer parameters generating the Fourier-enhanced basis, and a lower-level problem that fits the projection coefficients on this basis using the Levenberg--Marquardt (LM) method~\citep{levenberg1944method,marquardt1963algorithm}.
Our contributions are threefold. First, we formulate the lower-level problem as a nonlinear least-squares problem and solve it with the LM algorithm, which replaces the nonconvex problem with a sequence of strictly convex damped subproblems, each admitting a closed-form solution. The framework subsumes IFeF-PINN~\citep{wu2025iterative} as a special case for linear PDEs. Second, we show that the Fourier feature mapping applied to the learned latent representation induces an adaptive stationary kernel (Lemma~\ref{lem:kernel_convergence}), which interprets the adaptive basis learning as a form of deep kernel learning~\citep{wilson2016deep} and offers an explanation for its ability to mitigate spectral bias. Third, we establish global convergence of the alternating scheme without requiring convexity of the lower-level problem (Theorem~\ref{thm:classical}).

The organization of this paper is as follows.
Section~\ref{sec:background} provides a brief overview of PINNs and Fourier feature mapping.
Section~\ref{sec:method} introduces the FALM-PINN framework, including the upper- and lower-level problems, the global alternating algorithm, and its convergence guarantees. 
Section~\ref{sec:experiments} reports numerical experiments on linear and nonlinear benchmarks, and Section~\ref{sec:conclusion} summarizes the conclusions and discusses some promising future directions.

\section{Background}
\label{sec:background}
\subsection{Physics-informed neural networks}
PINNs incorporate the governing equations directly into the training objective, enabling data-driven learning under physical constraints~\citep{raissi2019physics,karniadakis2021physics}. Let $n > 0$ denote the dimension of the spatio-temporal domain, $m > 0$ the number of solution components, $M > 0$ the number of equations in the system, and $k > 0$ the highest order of differentiation in the system. We consider the following PDE system on a bounded domain
$\Omega \subset \mathbb{R}^n$:
\begin{equation}
\begin{aligned}
\label{equ:general_pde}
\mathscr{N}[u](x) &= f(x), \quad x \in \Omega, \\
 \mathscr{B}[u](x) &= g(x), \quad x \in \Gamma \subseteq \partial \Omega,
\end{aligned}
\end{equation}
where $\mathscr{N} \colon \mathcal{U} \to C(\Omega; \mathbb{R}^M)$ is a partial differential operator and $\mathscr{B} \colon \mathcal{U} \to C(\Gamma; \mathbb{R}^m)$ denotes the linear boundary operator, both defined on the classical solution space $\mathcal{U}\coloneqq C^k(\overline{\Omega}; \mathbb{R}^m)$. This formulation accommodates a wide range of boundary conditions, including Dirichlet, Neumann, Robin, and periodic types. The function $u \in \mathcal{U}$ is the solution, $f \in C(\Omega; \mathbb{R}^M)$ is the source term, and $g \in C(\Gamma; \mathbb{R}^m)$ specifies the boundary condition.

A standard PINN approximates $u$ by an MLP, denoted by $\hat{u}$, and is trained to minimize the composite loss:
\begin{equation}
\label{eq:pinn_sampled_loss}
\hat{\mathcal{L}}_\lambda(\hat{u}) = \frac{1}{N_b} \sum_{i=1}^{N_b} \|  \mathscr{B}[\hat{u}](x_b^i) - g(x_b^i)  \|^2 + \frac{\lambda}{N_f} \sum_{i=1}^{N_f} \|\mathscr{N}[\hat{u}](x_f^i)- f(x_f^i)\|^2,
\end{equation}
where $X_b = \{x_b^i\}_{i=1}^{N_b} \subset \Gamma$ and $X_f = \{x_f^i\}_{i=1}^{N_f} \subset \Omega$ are the boundary and collocation point sets, respectively. The physics weight $\lambda > 0$ balances the two loss terms and may be fixed or adapted during training.
To explicitly illustrate the network architecture, consider an $L$-hidden layer network defined as:
\begin{equation}
\label{equ:std_pinn}
\begin{aligned}z^{(0)} &= x, \\
    z^{(l)} &= \alpha \left( W^{(l)} z^{(l-1)} + c^{(l)} \right), \quad 1 \le l \le L, \\
    \hat{u}(x) &= W^{(L+1)} z^{(L)} + c^{(L+1)},
\end{aligned}
\end{equation}
where $\alpha$ is the activation function (e.g., $\tanh$), $W^{(l)} \in \mathbb{R}^{d_l \times d_{l-1}}$ and $c^{(l)} \in \mathbb{R}^{d_l}$ are the weight matrix and bias of layer $l$. All hidden layer parameters are collected into $\omega = \{W^{(l)}, c^{(l)}\}_{l=1}^{L}$, and the final hidden layer output is denoted $z_\omega(x) \triangleq z^{(L)} \in \mathbb{R}^p$, where $p = d_L$ 
is the width of the final hidden layer. This vector serves as a learned nonlinear latent representation of the input. 
Finally, $W^{(L+1)} \in \mathbb{R}^{m \times p}$ and 
$c^{(L+1)} \in \mathbb{R}^m$ are the output layer weight matrix and bias. In standard PINN training, all parameters are optimized jointly by 
minimizing $\hat{\mathcal{L}}_\lambda$ via gradient descent.
  
\subsection{Fourier feature mapping}
\label{sec:ff}
Random Fourier features, introduced by Rahimi and Recht~\citep{rahimi2007random},
approximate stationary kernels via an explicit feature map and have since been adopted across machine learning and scientific computing. The method rests on Bochner's theorem,
which characterizes any continuous, positive-definite stationary kernel as the Fourier transform
of a non-negative measure. Following Tancik et al.~\cite{tancik2020fourier},
the Fourier feature mapping $\gamma_{D}$ of an input coordinate $v\in\mathbb{R}^q$ is
\begin{equation}
  \gamma_{D}(v)=\frac{1}{\sqrt{D}}
  \begin{bmatrix}\cos(2\pi \mathbf{B}_Dv)\\ \sin(2\pi \mathbf{B}_Dv)\end{bmatrix}
  \in\mathbb{R}^{2D},
  \label{eq:ff-mapping}
\end{equation}
where the entries of $\mathbf{B}_D\in\mathbb{R}^{D\times q}$ are drawn
independently from a Gaussian distribution $\mathcal{N}(0,\sigma^{2})$, and the
bandwidth $\sigma>0$ controls the range of representable frequencies. 
Training an MLP on these embedded coordinates corresponds to kernel regression with a stationary kernel of tunable bandwidth~\citep{tancik2020fourier}, which
alleviates spectral bias and enables the network to learn high-frequency
components more efficiently. The same mapping has been applied to
the raw input coordinates in PINNs to
improve the approximation of multi-scale and high-frequency PDEs~\citep{wang2021eigenvector}.

\section{The FALM-PINN framework}
\label{sec:method}
Standard PINN training updates the hidden layer parameters $\omega$ and the linear output layer jointly within a single nonconvex objective, entangling feature learning and coefficient fitting. The latent features $z_\omega$ shift at every iteration while the output layer fits coefficients on a moving target. This coupling, aggravated by the spectral bias that prevents $z_\omega$ from representing high-frequency components, motivates training the two components separately.

A natural idea is to decouple training into two alternating subproblems: an upper-level problem that updates the hidden layer parameters $\omega$ for basis generation, and a lower-level problem that solves for the projection coefficients. This separation allows each problem to be addressed with a dedicated solver. To further mitigate spectral bias, we apply the Fourier feature mapping of Section~\ref{sec:ff} to $z_\omega$ rather than to the raw input coordinates, enriching the basis with high-frequency components on an adaptive feature space. On this enriched basis, the lower-level problem reduces to a regression that admits a closed-form solution for linear PDEs and is solved efficiently by linearization for nonlinear ones. We develop the two subproblems in Sections~\ref{sec:upper-level} and~\ref{sec:lower_level}, combine them into the global alternating algorithm in Section~\ref{sec:global_opt}, and collect the theoretical properties of the framework in Section~\ref{sec:theory}. A fully worked example on a family of scalar conservation laws is provided in \ref{app:example}. The integrated architecture is illustrated in Fig.~\ref{fig:FALM_structure}.

\begin{figure}[htbp]
    \centering
    \includegraphics[width=0.99\linewidth]{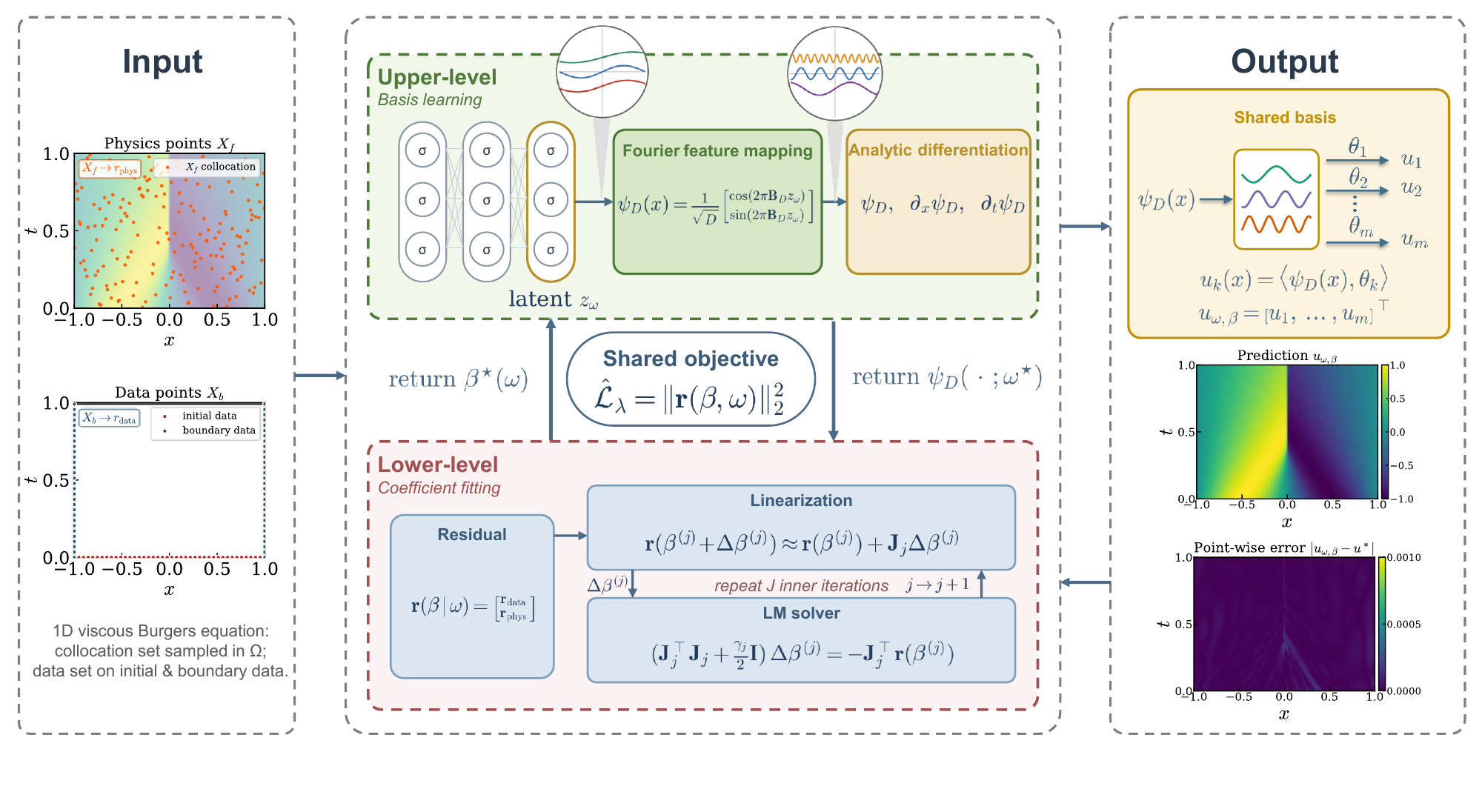}
    \caption{Architecture of FALM-PINN}
    \label{fig:FALM_structure}
\end{figure}

\subsection{The upper-level problem}
\label{sec:upper-level}
The linear output layer in a standard PINN induces a dot-product kernel 
$k(x,x') = z_\omega(x)^\top z_\omega(x')$ in the latent space, whose 
expressivity is determined entirely by the learned features $z_\omega$. 
To enrich this representation, we apply the Fourier feature mapping to
$z_\omega(x)$, replacing the dot-product kernel with a stationary kernel that can represent high-frequency components without increasing the network's depth or width.

Concretely, we apply the Fourier feature mapping $\gamma_D\colon \mathbb{R}^p \to \mathbb{R}^{2D}$ of Section~\ref{sec:ff} to the latent representation, and define the composite feature map $\psi_D\colon \mathbb{R}^n \to \mathbb{R}^{2D}$ as:
\begin{equation}
\label{eq:rff_mapping}
\psi_D(x) = \gamma_D\left(z_\omega(x)\right) 
= \frac{1}{\sqrt{D}} 
\begin{bmatrix} 
\cos\left(2\pi \mathbf{B}_D z_\omega(x)\right) \\ 
\sin\left(2\pi \mathbf{B}_D z_\omega(x)\right) 
\end{bmatrix},
\end{equation}
where $\mathbf{B}_D \in \mathbb{R}^{D \times p}$ is a fixed Fourier feature matrix with entries drawn independently from $\mathcal{N}(0,\sigma^2)$, and $\sigma > 0$ is a bandwidth hyperparameter that governs the frequency range of the induced kernel. We refer to $\psi_D(x)$ as the Fourier-enhanced features. As $D \to \infty$, the inner product $\langle \psi_D(x), \psi_D(x')\rangle$ converges almost surely to a Gaussian RBF kernel evaluated in the latent space (Lemma~\ref{lem:kernel_convergence}).

The approximated solution is $u_{\omega,\beta} = \left[ u_1, \dots, u_m\right]^{\top}$ with $m$ solution components. Projecting each component on the same Fourier-enhanced basis $\psi_D(x) \in \mathbb{R}^{2D}$, the final reconstructions will differ only in their projections' coefficients. The $k$-th component is therefore parameterized as $u_k(x) = \langle \psi_D(x), \theta_k \rangle$ with a trainable coefficient vector $\theta_k \in \mathbb{R}^{2D}$. Collecting the coefficients column-wise and stacking them into a single vector yields
\begin{equation}
\label{equ:u_beta_form}
\begin{aligned}
&\Theta = [\theta_1, \ldots, \theta_m] \in \mathbb{R}^{2D \times m}, \\
    &\beta = \operatorname{vec}(\Theta)=[\theta_1^\top, \theta_2^\top, \ldots, \theta_m^\top]^\top
\in \mathbb{R}^{2mD},\\
&u_{\omega,\beta}(x) =
\Theta^\top \psi_D(x),
\quad x\in\Omega .
\end{aligned}
\end{equation}

\begin{remark}[Connection to deep kernel learning]
The composite map $\psi_D = \gamma_D \circ z_\omega$ can be 
interpreted as a deep kernel learning architecture~\citep{wilson2016deep}. The hidden layers parameterized by $\omega$ learn a nonlinear embedding of the input space, while $\gamma_D$ induces a shift-invariant kernel in this learned latent space. The upper-level update of $\omega$ therefore adapts the kernel geometry to the structure of the PDE solution.
\end{remark}

Formally, with the coefficient vector $\beta$ held fixed, the upper-level problem minimizes the shared training objective over $\omega$ using a gradient-based optimizer (e.g., Adam~\cite{kingma2014adam}):
\begin{equation}
\label{equ:upper_level}
\omega^\star(\beta) = \arg\min_\omega
\hat{\mathcal{L}}_\lambda\bigl(u_{\omega,\beta}\bigr) \coloneqq \arg\min_\omega\mathcal{L}_{\mathrm{upper}}(\omega \mid \beta).
\end{equation}
The lower-level problem minimizes the same objective $\hat{\mathcal{L}}_\lambda$ with $\omega$ fixed, recast as an equivalent nonlinear least-squares problem over $\beta$ (Section~\ref{sec:lower_level}). The two updates are then alternated, as formalized in Section~\ref{sec:global_opt}.

\subsection{The lower-level problem}
\label{sec:lower_level}
With the Fourier-enhanced features $\psi_D$ fixed for a given $\omega$, the lower-level problem finds the optimal projection coefficients $\beta$ that satisfy the governing PDE and boundary conditions. This decoupling exposes a least-squares structure in the objective, which we exploit with the LM solver developed in Section~\ref{sec:lm_linearization}.

The PINN loss in Eq.~\eqref{eq:pinn_sampled_loss} has the form of a nonlinear least-squares problem. To set up the iterative solver of Section~\ref{sec:lm_linearization}, we define a composite residual vector. Recall that the PDE system in Eq.~\eqref{equ:general_pde} comprises $M$ equations and $m$ solution components. We construct the vectorized boundary data
$\bm{y}_b \in \mathbb{R}^{m N_b}$ and source term
$\bm{f} \in \mathbb{R}^{M N_f}$ by stacking over state variables and PDE equations, respectively:
\begin{equation}
\label{eq:yb_def}
\bm{y}_b = \begin{bmatrix} \bm{y}_b^{(1)} \\ \vdots \\ \bm{y}_b^{(m)} \end{bmatrix}, \quad
\bm{f} = \begin{bmatrix} \bm{f}^{(1)} \\ \vdots \\ \bm{f}^{(M)} \end{bmatrix},
\end{equation}
where each sub-vector is defined as $\bm{y}_b^{(k)} = \big[ g_k(x_b^1), \ldots, g_k(x_b^{N_b}) \big]^\top \in \mathbb{R}^{N_b}$ for $k = 1, \dots, m$ and $\bm{f}^{(j)} = \big[ f_j(x_f^1), \ldots, f_j(x_f^{N_f}) \big]^\top \in \mathbb{R}^{N_f}$ for $j = 1, \dots, M$.

Leveraging the linearity of the boundary operator $\mathscr{B}$, we construct a matrix representation of the boundary residual
$\bm{M}_b(\omega) \in \mathbb{R}^{mN_b \times 2mD}$ as:
\begin{equation}
\label{eq:Mb_def}
\bm{M}_b(\omega) = \begin{bmatrix} \bm{M}_b^{(1)}(\omega) \\ \vdots \\ \bm{M}_b^{(m)}(\omega) \end{bmatrix}, \quad
\bm{M}_b^{(k)}(\omega) = \begin{bmatrix}
\mathscr{B}_k[\psi_D](x_b^1)^\top \\
\vdots \\
\mathscr{B}_k[\psi_D](x_b^{N_b})^\top
\end{bmatrix} \in \mathbb{R}^{N_b \times 2mD},
\end{equation}
where $\mathscr{B}_k[\psi_D](x_b^i)$ represents the boundary operator applied to the $k$-th solution component, evaluated at point $x_b^i$, for $i = 1, \dots, N_b$.

For the physics residual, the potential nonlinearity of the PDE operator $\mathscr{N}$ makes the mapping from $\beta$ to the residual nonlinear. We collect the stacked physics terms into the nonlinear operator $\bm{M}_f(\beta \mid \omega) \in \mathbb{R}^{M N_f}$, written in block-vector form:
\begin{equation}
\label{eq:Mf_def}
\bm{M}_f(\beta \mid \omega) = \begin{bmatrix} \bm{M}_f^{(1)}(\beta \mid \omega) \\ \vdots \\ \bm{M}_f^{(M)}(\beta \mid \omega) \end{bmatrix}, \quad
\bm{M}_f^{(j)}(\beta \mid \omega) = \begin{bmatrix}
\mathscr{N}_j[u_{\omega,\beta}](x_f^1) \\
\vdots \\
\mathscr{N}_j[u_{\omega,\beta}](x_f^{N_f})
\end{bmatrix} \in \mathbb{R}^{N_f},
\end{equation}
where $\mathscr{N}_j[u_{\omega,\beta}](x_f^i)$ denotes the $j$-th equation of the PDE system evaluated at $x_f^i$, for $i = 1, \dots, N_f$. 
Substituting these formulations into the PINN objective yields the composite residual vector
\begin{equation}
\label{eq:residual_general}
\bm{r}(\beta \mid \omega)
= \begin{bmatrix}
\bm{r}_{\mathrm{data}}(\beta \mid \omega) \\
\bm{r}_{\mathrm{phys}}(\beta \mid \omega)
\end{bmatrix}
\in \mathbb{R}^{m N_b + M  N_f},
\end{equation}
where
\begin{equation}
\label{eq:residual_components}
\begin{aligned}
\bm{r}_{\mathrm{data}}(\beta \mid \omega)
&= \frac{1}{\sqrt{N_b}}
   \bigl(\bm{M}_b(\omega)\beta - \bm{y}_b\bigr)
   \in \mathbb{R}^{m N_b}, \\
\bm{r}_{\mathrm{phys}}(\beta \mid \omega)
&= \sqrt{\frac{\lambda}{N_f}}
   \bigl(\bm{M}_f(\beta \mid \omega) - \bm{f}\bigr)
   \in \mathbb{R}^{M N_f}.
\end{aligned}
\end{equation}
Under this setting, the lower-level objective is a nonlinear least-squares reformulation of $\hat{\mathcal{L}}_\lambda$ over $\beta$ for any fixed $\omega$:
\begin{equation}
\label{equ:NLS_pinnloss}
\mathcal{L}_{\mathrm{lower}}(\beta \mid \omega)
= \|\bm{r}(\beta\mid\omega)\|_2^2. 
\end{equation}
The optimal coefficient vector $\beta^\star(\omega)$ is then obtained by minimizing:
\begin{equation}
\label{equ:ll_obj}
    \beta^\star(\omega) = \arg\min_\beta 
\hat{\mathcal{L}}_\lambda\bigl(u_{\omega,\beta}\bigr) \coloneqq \arg\min_{\beta}\mathcal{L}_{\mathrm{lower}}(\beta \mid \omega).
\end{equation}
Since $\bm{r}_{\mathrm{phys}}$ is generally nonlinear in $\beta$ through $\bm{M}_f(\beta\mid\omega)$, the lower-level objective is nonconvex and admits no closed-form minimizer. We propose to solve it iteratively in Section~\ref{sec:lm_linearization}.

\subsubsection{Linearization via the Levenberg--Marquardt algorithm}
\label{sec:lm_linearization}
We solve the lower-level problem with the Levenberg--Marquardt (LM) algorithm~\cite{levenberg1944method,marquardt1963algorithm}, which iteratively linearizes the residual and solves a sequence of damped least-squares subproblems. Since the upper-level parameters~$\omega$ remain fixed throughout the lower-level solve, we write the residual vector and the stacked physics operator as $\bm{r}(\beta)\triangleq\bm{r}(\beta\mid\omega)$ and $\bm{M}_f(\beta )\triangleq \bm{M}_f(\beta\mid \omega)$.

\paragraph{First-order linearization}
Let $\beta^{(j)}$ denote the current coefficient estimate at the $j$-th LM iteration. We linearize the residual vector around $\beta^{(j)}$ via a first-order Taylor expansion:
\begin{equation}
\begin{aligned}
\label{equ:lm_taylor}
&\bm{r}(\beta^{(j)} + \Delta\beta^{(j)})
\approx \bm{r}(\beta^{(j)}) + \mathbf{J}_j\Delta\beta^{(j)},
\end{aligned}
\end{equation}
where $\Delta\beta^{(j)}\in\mathbb{R}^{2mD}$ is the update and
$\mathbf{J}_j = \nabla_\beta\bm{r}(\beta)\big|_{\beta=\beta^{(j)}} \in\mathbb{R}^{(mN_b+MN_f)\times 2mD}$ is the Jacobian of the residual vector with the block form $\mathbf{J}_j
=[\mathbf{J}_{\mathrm{data}}^\top, {{\mathbf{J}_\mathrm{phys}}^{(j)}}^\top]^\top$,
\begin{equation}
\begin{aligned}
\label{equ:jacobian_blocks}
\mathbf{J}_{\mathrm{data}}
&= \frac{1}{\sqrt{N_b}}\bm{M}_b(\omega)
  \in\mathbb{R}^{mN_b\times 2mD},\\
{\mathbf{J}_\mathrm{phys}}^{(j)}
&= \sqrt{\frac{\lambda}{N_f}}
  \nabla_\beta\bm{M}_f(\beta^{(j)})
  \in\mathbb{R}^{MN_f\times 2mD}.
  \end{aligned}
\end{equation}
Here, $\mathbf{J}_{\mathrm{data}}$ is a constant matrix
across iterations since $\bm{r}_{\mathrm{data}}$ is linear in $\beta$.

\paragraph{Damped least-squares subproblem}
To iteratively minimize the nonlinear
least-squares objective $\mathcal{L}_{\mathrm{lower}}(\beta \mid \omega)$ in Eq.~\eqref{equ:NLS_pinnloss}, we substitute the linearized residual~\eqref{equ:lm_taylor} into the objective and augment with a damping term, obtaining the local quadratic subproblem:
\begin{equation}
\label{equ:lm_subproblem}
\min_{\Delta\beta^{(j)}}
Q_j(\Delta\beta^{(j)})
= \bigl\|\bm{r}(\beta^{(j)})+\mathbf{J}_j\Delta\beta^{(j)}\bigr\|_2^2
  + \frac{\gamma_j}{2}\|\Delta\beta^{(j)}\|_2^2,
\end{equation}
where $\gamma_j>0$ is the damping parameter at iteration~$j$. Setting
$\nabla_{\Delta\beta}Q_j = 0$ gives the closed-form solution of the subproblem:
\begin{equation}
\label{equ:lm_normal_eq}
\bigl(\mathbf{J}_j^\top\mathbf{J}_j
      + \tfrac{\gamma_j}{2}\mathbf{I}\bigr)\Delta\beta^{(j)}
= -\mathbf{J}_j^\top\bm{r}(\beta^{(j)}).
\end{equation}
The dynamic damping parameter $\gamma_j>0$ ensures positive definiteness of the subproblem and balances the trade-off between the Gauss--Newton and steepest descent directions. For small $\gamma_j$, the
update approaches the Gauss--Newton step for fast convergence, where
$\Delta\beta^{(j)}\approx-(\mathbf{J}_j^\top\mathbf{J}_j)^{-1}
\mathbf{J}_j^\top\bm{r}(\beta^{(j)})$. In contrast, for large $\gamma_j$, the step is dominated by the gradient, reducing to the steepest descent direction $\Delta\beta^{(j)}\approx-(2/\gamma_j)\mathbf{J}_j^\top\bm{r}(\beta^{(j)})$, which guarantees a strict decrease in $\mathcal{L}_{\mathrm{lower}}$.

\paragraph{Summary of the lower-level solver}
Each lower-level update solves a sequence of $J$ damped subproblems in Eq.~\eqref{equ:lm_subproblem}. The optimal projection coefficients $\beta^\star(\omega)$, defined in Eq.~\eqref{equ:ll_obj}, are approximated by the LM output  $ \beta^{(J)}$ computed iteratively as:
\begin{equation}
\label{eq:lower-level}
\begin{aligned}
  \Delta\beta^{(j)}&=\arg\min_{\Delta\beta} Q_j(\Delta\beta),\\
  \beta^{(j+1)}&=\beta^{(j)}+\Delta\beta^{(j)},
  \quad j = 0, \dots, J-1,
\end{aligned}
\end{equation}
with the initial $\beta^{(0)}$ starting from the previous lower-level solve. Each subproblem is a strictly convex quadratic problem with a unique closed-form solution given in Eq.~\eqref{equ:lm_normal_eq}. Moreover, since the linearization error is quadratic in the residual (Lemma~\ref{thm:lm-linearization-error}), the model is accurate with good initialization. Empirically, one or two iterations ($J \in \{1, 2\}$) are sufficient to obtain an accurate estimate of~$\beta^\star(\omega)$. The per-iteration cost is dominated by computing the Gram matrix $\mathbf{J}_j^\top\mathbf{J}_j$, which
requires evaluating the physics Jacobian
$\nabla_\beta\bm{M}_f(\beta^{(j)})$. Since both the PDE operator
$\mathscr{N}$ and the Fourier-enhanced feature map $\psi_D$ are known in closed form, the Jacobian rows can be constructed analytically, as illustrated for the scalar conservation laws example in~\ref{app:example}.

When the PDE operator $\mathscr{N}$ is linear in $u$, the physics residual $\bm{r}_{\mathrm{phys}}$ becomes affine in $\beta$, and therefore the linearization~\eqref{equ:lm_taylor} is exact. As the Jacobian $\mathbf{J}$ is independent of $\beta^{(j)}$, the lower-level problem reduces to a strictly convex Kernel Ridge Regression (KRR) problem where $\mathcal{L}_{\mathrm{lower}}$ is a convex quadratic in $\beta$. Consequently, the LM iteration converges in a single step to its unique global minimizer~\cite{wu2025iterative}.

\subsection{Global alternating optimization}
\label{sec:global_opt}
Combining the upper- and lower-level problems introduced above, we formulate the overall training as the following alternating optimization:
\begin{equation}
\label{equ:joint_problem}
\begin{aligned}
\omega^\star(\beta) &= \arg\min_\omega 
\hat{\mathcal{L}}_\lambda\bigl(u_{\omega,\beta}\bigr) \coloneqq \arg\min_\omega\mathcal{L}_{\mathrm{upper}}(\omega \mid \beta), \\
  \beta^\star(\omega) &= \arg\min_\beta 
\hat{\mathcal{L}}_\lambda\bigl(u_{\omega,\beta}\bigr) \coloneqq \arg\min_{\beta}\mathcal{L}_{\mathrm{lower}}(\beta \mid \omega).
\end{aligned}
\end{equation}
A practical challenge is that the lower-level solver requires both an expressive feature basis~$\psi_D$ and an initial coefficient estimate close to $\beta^\star(\omega)$. As shown in Lemma~\ref{thm:lm-linearization-error}, the LM linearization error scales with the squared residual, so a poor initialization yields inaccurate steps. When $\omega$ and $\beta$ are both initialized randomly, the training may converge slowly or settle at a poor local minimum.
To address this, we precede the training with a
joint warm-up phase of $N_{\mathrm{warm}}$ iterations, during
which $\omega$ and $\beta$ are updated together on the objective $\hat{\mathcal{L}}_\lambda$:
\begin{equation}
\label{equ:warmup_update}
\omega_{k+1}   = \omega_k   
               - \eta_\omega
                 \gradw\hat{\mathcal{L}}_{\lambda}
                 (u_{\omega_k,\beta_k}), 
\qquad
\beta_{k+1}    = \beta_k    
               - \eta_\beta
                 \nabla_\beta\hat{\mathcal{L}}_{\lambda}
                 (u_{\omega_k,\beta_k}).
\end{equation}
Algorithm~\ref{alg:ifef} summarizes the full training procedure, which consists of two phases. Phase~1 performs the joint warm-up of $(\omega,\beta)$, providing an accurate initialization. Phase~2 then applies the alternating optimization in Eq.~\eqref{equ:joint_problem}.

\begin{algorithm}[htp]
\caption{FALM-PINN}
\label{alg:ifef}
\begin{algorithmic}
\STATE \textbf{Input:} collocation sets $X_b$, $X_f$; 
       hyperparameters $\lambda$, 
       $\sigma$, $D$, $J$, $N_{\mathrm{warm}}$, 
       $N_{\mathrm{alter}}$,
       $\eta_\omega$, $\eta_\beta$
\STATE \textbf{Initialize} parameters $\omega_0$,  $\beta_0$, 
       and fixed Fourier feature matrix $\mathbf{B}_D \sim \mathcal{N}(0,\sigma^2 I)$

\vspace{0.4em}
\STATE {\(\triangleright\) \textit{Phase 1: Joint warm-up}}
\FOR{$k = 0$ \TO $N_{\mathrm{warm}}-1$}
    \STATE Compute Fourier-enhanced features 
           $\psi_D(\cdot) = \gamma_D(z_{\omega_k}(\cdot))$ 
    \STATE Joint warm-up:
    \[
        \omega_{k+1} = \omega_k 
                     - \eta_\omega
                       \gradw\hat{\mathcal{L}}_\lambda
                       (u_{\omega_k,\beta_k}),\quad
        \beta_{k+1}  = \beta_k   
                     - \eta_\beta
                       \nabla_\beta\hat{\mathcal{L}}_\lambda
                       (u_{\omega_k,\beta_k})
    \]
\ENDFOR

\vspace{0.4em}
\STATE {\(\triangleright\) \textit{Phase 2: Alternating optimization}}
\FOR{$k = N_{\mathrm{warm}}$ \TO $N_{\mathrm{warm}}+N_{\mathrm{alter}}-1$}
    \STATE Compute Fourier-enhanced features
           $\psi_D(\cdot) = \gamma_D(z_{\omega_k}(\cdot))$
    \STATE \textbf{Lower-level update}:
    \STATE Set inner solver initialization: $\beta^{(0)} \leftarrow \beta_k$
    \FOR{$j = 0$ \TO $J-1$}
            \STATE Compute Jacobian $\mathbf{J}_j$ 
                   via Eq.~\eqref{equ:jacobian_blocks}
            \STATE Solve normal equation 
                   $\bigl(\mathbf{J}_j^\top\mathbf{J}_j 
                          + \tfrac{\gamma_j}{2}\mathbf{I}
                   \bigr)\Delta\beta^{(j)} 
                   = -\mathbf{J}_j^\top \bm{r}(\beta^{(j)})$
            \STATE $\beta^{(j+1)} = \beta^{(j)} + \Delta\beta^{(j)}$
    \ENDFOR
    \STATE Set $\beta_{k+1} \leftarrow \beta^{(J)}$
    \STATE \textbf{Upper-level update}:
    \[
        \omega_{k+1} = \omega_k 
                     - \eta_\omega
                       \nabla_\omega\hat{\mathcal{L}}_\lambda
                       (u_{\omega_k,\beta_{k+1}})
    \]
\ENDFOR
\RETURN $\omega_{N_{\mathrm{warm}}+N_{\mathrm{alter}}}$,
        $\beta_{N_{\mathrm{warm}}+N_{\mathrm{alter}}}$
\end{algorithmic}
\end{algorithm}

\begin{remark}[Relation to IFeF-PINN]
    FALM-PINN and IFeF-PINN~\citep{wu2025iterative} share the same decoupled architecture: an upper-level problem learns a Fourier-enhanced basis on the last hidden layer, and a lower-level problem fits the output coefficients on this basis. 
    FALM-PINN differs from IFeF-PINN in applicability, optimization formulation, and convergence guarantees.
    IFeF-PINN is developed for scalar linear PDEs, for which the lower-level problem is convex and admits a unique global minimizer. It is formulated as a bi-level optimization but solved by iterative training, and its convergence guarantee rests on this convexity.  
    FALM-PINN handles general nonlinear PDE systems, where the lower-level problem is a nonconvex nonlinear least-squares problem over the stacked coefficients of all components, solved by the LM algorithm. By formulating the alternating optimization on the shared objective $\hat{\mathcal{L}}_\lambda$, we establish convergence to critical points (Theorem~\ref{thm:classical}) without requiring the convexity of the lower-level problem.
\end{remark}

\subsection{Theoretical analysis}
\label{sec:theory}
This section collects the theoretical properties of FALM-PINN. We first formalize the kernel interpretation of the Fourier-enhanced basis, then bound the error of the LM linearization underlying the lower-level solver, and finally establish global convergence of the alternating
optimization.

\subsubsection{Kernel approximation in the latent space}
Recall the composite feature map $\psi_D = \gamma_D \circ z_\omega$ of Section~\ref{sec:upper-level}. The following lemma formalizes its connection to a Gaussian kernel in the latent space.

\begin{lemma}[Kernel approximation]
\label{lem:kernel_convergence}
Let $k_{\mathrm{RBF}}(z, z') = \exp\left(-2\pi^2\sigma^2\|z - z'\|_2^2\right)$ be a Gaussian RBF kernel defined on the latent space $\mathbb{R}^p$. For any spatio-temporal inputs $x, x' \in \mathbb{R}^n$, let $z = z_\omega(x)$ and $z' = z_\omega(x')$. Then the inner product of the explicit feature maps converges to the exact kernel almost surely:
\begin{equation}
\lim_{D \to \infty} \langle \psi_D(x), \psi_D(x') \rangle = k_{\mathrm{RBF}}(z, z').
\end{equation}
\end{lemma}
\begin{proof}
The inner product of the feature vectors is
\begin{equation}
\langle \psi_D(x), \psi_D(x') \rangle 
= \frac{1}{D} \sum_{j=1}^{D} \left[ 
  \cos(2\pi \mathbf{b}_j^\top z)\cos(2\pi \mathbf{b}_j^\top z') 
  + \sin(2\pi \mathbf{b}_j^\top z)\sin(2\pi \mathbf{b}_j^\top z') 
\right],
\end{equation}
where $\mathbf{b}_j^\top$ is the $j$-th row of $\mathbf{B}_D$. Applying 
the identity 
$\cos\alpha\cos\beta + \sin\alpha\sin\beta = \cos(\alpha - \beta)$ yields:
\begin{equation}
\label{eq:inner_product_cos}
\langle \psi_D(x), \psi_D(x') \rangle 
= \frac{1}{D} \sum_{j=1}^{D} 
  \cos\left(2\pi \mathbf{b}_j^\top (z - z')\right).
\end{equation}
Since $\mathbf{b}_j \sim \mathcal{N}(0, \sigma^2 I_p)$, the expectation 
of each summand is the characteristic function of the Gaussian distribution 
evaluated at the displacement $\delta = z - z'$:
\begin{equation}
\label{equ:rbf_lemma}
\begin{aligned}
\mathbb{E}\left[\cos(2\pi \mathbf{b}_j^\top \delta)\right] 
= \operatorname{Re}\left(\mathbb{E}\left[
  e^{i2\pi \mathbf{b}_j^\top \delta}\right]\right) = \exp\left(-2\pi^2\sigma^2 \|\delta\|_2^2\right).
\end{aligned}
\end{equation}
Since the $\mathbf{b}_j$ are i.i.d.\ and 
$|\cos(\cdot)| \leq 1$, the strong law of large numbers applied to the 
sample mean in~\eqref{eq:inner_product_cos} gives:
\begin{equation*}
\frac{1}{D}\sum_{j=1}^{D} 
\cos\left(2\pi \mathbf{b}_j^\top \delta\right) 
\xrightarrow{D \to \infty} 
\mathbb{E}\left[\cos\left(2\pi \mathbf{b}_j^\top \delta\right)\right] 
= k_{\mathrm{RBF}}(z, z') \quad \text{almost surely.}
\end{equation*}
\end{proof}

\subsubsection{Accuracy of the LM linearization}
We next quantify the accuracy of the first-order
linearization~\eqref{equ:lm_taylor}. The following bound shows that the
linearization error is quadratic in the residual, which justifies the small number of inner iterations used in practice and motivates the
warm-up phase of Section~\ref{sec:global_opt}.

\begin{lemma}[Linearization error]
\label{thm:lm-linearization-error}
Assume the following:
\begin{enumerate}
  \item $\mathscr{N}$ is twice continuously differentiable in $u$.
  \item The lower-level iterates lie in a compact set $K$ on which the residual
    Jacobian $\mathbf{J}_j$ has full column rank, with the smallest singular
    value bounded below by $\underline{\sigma}>0$.
  \item $M_{\mathscr{N}}<\infty$ is a uniform bound over $K$ on the curvature of
    the PDE operator, with
    $\|\nabla^2_\beta \mathscr{N}_j[u_{\omega,\beta}](x_f^i)\|_2 \le M_{\mathscr{N}}$.
\end{enumerate}
Then the LM step~\eqref{equ:lm_normal_eq} satisfies
$\|\Delta\beta^{(j)}\|_2\le\underline{\sigma}^{-1}\|\bm{r}(\beta^{(j)})\|_2$,
and the linearized residual~\eqref{equ:lm_taylor} incurs the error
\begin{equation}
\begin{aligned}
    \bigl\|\bm{r}(\beta^{(j)}+\Delta\beta^{(j)})
   -\bigl[\bm{r}(\beta^{(j)})+\mathbf{J}_j\Delta\beta^{(j)}\bigr]\bigr\|_2
&\le \tfrac12 M_{\mathscr{N}}\sqrt{M\lambda}\,\|\Delta\beta^{(j)}\|_2^2 \\
&\le \frac{M_{\mathscr{N}}\sqrt{M\lambda}}{2\underline{\sigma}^{2}}
        \|\bm{r}(\beta^{(j)})\|_2^2 .
\end{aligned}
\label{eq:lm-lin-err}
\end{equation}
\end{lemma}
\begin{proof}
The data residual is affine in $\beta$, so its linearization is exact, and the
error is therefore confined to the $MN_f$ physics components. Each physics component has the form
$r_{\mathrm{phys},\ell}(\beta) = \sqrt{\lambda/N_f}\,(\mathscr{N}_j[u_{\omega,\beta}](x_f^i)
- f_j(x_f^i))$ for some equation--point pair $(j,i)$ and is $C^2$ in $\beta$.
By the Taylor--Lagrange formula~\cite{coleman2013calculus}, there exists $\xi_\ell$ such that
\begin{equation}
\label{eq:lm-lin-taylor}
\begin{aligned}
\bigl|r_{\mathrm{phys},\ell}(\beta + \Delta\beta) - r_{\mathrm{phys},\ell}(\beta)
   - \nabla_\beta r_{\mathrm{phys},\ell}(\beta)^\top \Delta\beta\bigr|
&= \tfrac12\bigl|\Delta\beta^\top \nabla^2_\beta r_{\mathrm{phys},\ell}(\xi_\ell)\,\Delta\beta\bigr| \\
&\le \tfrac12\sqrt{\tfrac{\lambda}{N_f}}\,M_{\mathscr{N}}\,\|\Delta\beta\|_2^2.
\end{aligned}
\end{equation}
Summing over the $M N_f$ physics components gives the first inequality in~\eqref{eq:lm-lin-err}. The factor $1/\sqrt{N_f}$ in each component cancels the $\sqrt{N_f}$ growth from the $M N_f$ components, leaving the bound independent of $N_f$.

For the step bound, each subproblem~\eqref{equ:lm_subproblem} is a strictly
convex quadratic and admits a closed-form solution
\begin{equation}
\label{eq:lm-step-closed}
\Delta\beta^{(j)} = -\bigl(\mathbf{J}_j^\top\mathbf{J}_j
   + \tfrac{\gamma_j}{2}\mathbf{I}\bigr)^{-1}\mathbf{J}_j^\top\bm{r}(\beta^{(j)}) .
\end{equation}
Writing the singular value decomposition (SVD) $\mathbf{J}_j = U\Sigma V^\top$ with singular values $\{\sigma_i\}$, the operator in~\eqref{eq:lm-step-closed} can be written as
\begin{equation}
\label{eq:lm-step-svd}
\bigl(\mathbf{J}_j^\top\mathbf{J}_j + \tfrac{\gamma_j}{2}\mathbf{I}\bigr)^{-1}\mathbf{J}_j^\top
= V\,\operatorname{diag}\!\Bigl(\tfrac{\sigma_i}{\sigma_i^2+\gamma_j/2}\Bigr)U^\top ,
\end{equation}
whose singular values satisfy
\begin{equation}
\label{eq:lm-step-svalbound}
\frac{\sigma_i}{\sigma_i^2+\gamma_j/2} \;\le\; \frac{1}{\sigma_i}
   \;\le\; \frac{1}{\underline{\sigma}} \qquad\text{for any } \gamma_j\ge 0 .
\end{equation}
Taking the operator norm in~\eqref{eq:lm-step-closed} therefore gives
\begin{equation}
\label{eq:lm-step-bound}
\|\Delta\beta^{(j)}\|_2 \;\le\; \underline{\sigma}^{-1}\,\|\bm{r}(\beta^{(j)})\|_2 .
\end{equation}
Substituting into the first inequality of~\eqref{eq:lm-lin-err}, evaluated at $\beta=\beta^{(j)}$ and
$\Delta\beta=\Delta\beta^{(j)}$ yields the second inequality. Finally, if $\mathscr{N}$ is linear then $\nabla^2_\beta \mathscr{N}_j \equiv \mathbf{0}$, hence $M_{\mathscr{N}}=0$ and the error vanishes.
\end{proof}

\subsubsection{Convergence of the alternating optimization}
We establish convergence guarantees for the alternating optimization in Phase~2 of Algorithm~\ref{alg:ifef}. We denote $z = (\omega, \beta)$ and $\Phi(z) := \hat{\mathcal{L}}_\lambda(u_{\omega,\beta})$ throughout this section, and analyze the sequence $\{z^k = (\omega^k, \beta^k)\}$ generated by Phase~2. We first state the key properties required for our convergence analysis.

\begin{assumption}[Bounded iterates and joint smoothness]
\label{asm:bounded}
The sequence $\{z^k\}_{k \ge 0}$ remains in a compact set: there exists $\mathcal{B}$ such that $z^k \in \mathcal{B}$ for all $k \ge 0$.
\end{assumption}
 Since $\Phi$ is a composition of smooth activation functions and a quadratic residual loss, $\Phi$ is continuously differentiable on $\mathcal{B}$ and $\nabla^2 \Phi$ is bounded there. Consequently, $\Phi$ is $L$-smooth on $\mathcal{B}$ with $L = \sup_{z \in \mathcal{B}} \|\nabla^2 \Phi(z)\|_2 < \infty$, i.e.,
\[
\|\nabla \Phi(z) - \nabla \Phi(z')\| \le L \|z - z'\|,
\qquad \forall z, z' \in \mathcal{B}.
\]

\begin{proposition}[Inexact LM descent condition]
\label{prop:lm-descent}
Every lower-level update at iteration $k$ consists of $J$ damped subproblems of the form Eq.~\eqref{equ:lm_subproblem}, with $\beta^{(0)} = \beta^k$ and $\beta^{(J)} = \beta^{k+1}$. There exist constants $c_\beta, b_\beta > 0$ such that the overall update satisfies:
\begin{align}
\Phi(\omega^k, \beta^{k+1})
&\le \Phi(\omega^k, \beta^k) - c_\beta\|\beta^{k+1} - \beta^k\|^2,
\label{eq:lm-suff-decrease} \\
\|\nabla_\beta \Phi(\omega^k, \beta^{k+1})\|
&\le b_\beta\|\beta^{k+1} - \beta^k\|.
\label{eq:lm-approx-stat}
\end{align}
\end{proposition}

\begin{proof}
For brevity, we fix $\omega^k$ and write $\Phi(\cdot) \equiv \Phi(\omega^k, \cdot)$. From the closed-form LM update~\eqref{equ:lm_subproblem}, the gradient identity holds:
\begin{equation*}
\label{eq:grad-lm}
\nabla_\beta \Phi(\beta^{(j)}) = 2\mathbf{J}_j^\top \bm{r}(\beta^{(j)}) = -(2 \mathbf{J}_j^\top \mathbf{J}_j + \gamma \mathbf{I})\Delta\beta^{(j)}.
\end{equation*}

Since $\Phi(\beta)$ is $L$-smooth from Assumption~\ref{asm:bounded}, 
\begin{equation*}
\label{eq:descent-lemma-beta}
\begin{aligned}
    \Phi(\beta^{(j+1)}) &\le \Phi(\beta^{(j)}) + \langle \nabla_\beta \Phi(\beta^{(j)}), \Delta\beta^{(j)}\rangle + \tfrac{L}{2}\|\Delta\beta^{(j)}\|^2,\\
    & \le \Phi(\beta^{(j)}) + \langle -(2 \mathbf{J}_j^\top \mathbf{J}_j + \gamma \mathbf{I})\Delta\beta^{(j)}, \Delta\beta^{(j)}\rangle + \tfrac{L}{2}\|\Delta\beta^{(j)}\|^2,\\
    & \le \Phi(\beta^{(j)}) -2\|\mathbf{J}_j \Delta\beta^{(j)}\|^2 - \big(\gamma - \tfrac{L}{2}\big)\|\Delta\beta^{(j)}\|^2,\\
    & \le \Phi(\beta^{(j)}) - \big(\gamma - \tfrac{L}{2}\big)\|\Delta\beta^{(j)}\|^2.
\end{aligned}
\end{equation*}
With damping satisfying $\gamma > \frac{L}{2}$, we have the sufficient decrease condition for the single subproblem with $\tilde{c}_\beta \coloneqq  \gamma - \frac{L}{2} > 0$. Summing over $j = 0, 1, \ldots, J-1$:
\begin{equation*}
\Phi(\beta^{(J)}) \le \Phi(\beta^{(0)}) - \tilde{c}_\beta\sum_{j=0}^{J-1} \|\beta^{(j+1)} - \beta^{(j)}\|^2.
\end{equation*}
By the Cauchy-Schwarz inequality applied to $\beta^{k+1} - \beta^k = \sum_{j=0}^{J-1} \Delta\beta^{(j)}$:
\begin{equation*}
\|\beta^{k+1} - \beta^k\|^2 = \big\|\sum_{j=0}^{J-1} \Delta\beta^{(j)}\big\|^2 \le J\sum_{j=0}^{J-1} \|\Delta\beta^{(j)}\|^2,
\end{equation*}
We have the sufficient decrease condition for the lower-level problem as:
\begin{equation*}
\Phi(\omega^k, \beta^{k+1}) \le \Phi(\omega^k, \beta^k) - \tfrac{\tilde{c}_\beta}{J}\|\beta^{k+1} - \beta^k\|^2,
\end{equation*}
Setting $c_\beta := \tilde{c}_\beta/J$ proves Eq.~\eqref{eq:lm-suff-decrease}.

Next, we evaluate $\nabla_\beta \Phi(\beta^{(J)})$ by decomposition:
\begin{equation*}
\begin{aligned}
\|\nabla_\beta \Phi(\beta^{(J)})\| 
&= \|\nabla_\beta \Phi(\beta^{(J)}) - \nabla_\beta \Phi(\beta^{(J-1)}) + \nabla_\beta \Phi(\beta^{(J-1)})\| \\
&\le \|\nabla_\beta \Phi(\beta^{(J)}) - \nabla_\beta \Phi(\beta^{(J-1)})\| + \|(2 \mathbf{J}_{J-1}^\top \mathbf{J}_{J-1} + \gamma \mathbf{I})\Delta\beta^{(J-1)}\| \\
&\le L\|\Delta\beta^{(J-1)}\| + \big(2\|\mathbf{J}_{J-1}^\top \mathbf{J}_{J-1}\|_2 + \gamma\big)\|\Delta\beta^{(J-1)}\|,\\
&\le \big(L + 2\|\mathbf{J}_{J-1}^\top \mathbf{J}_{J-1}\|_2 + \gamma\big)\|\Delta\beta^{(J-1)}\| 
\end{aligned}
\end{equation*}

Since the sequence $\{z^k\}_{k \ge 0}$ is bounded on a compact set, we have $\|\mathbf{J}_{J-1}^\top \mathbf{J}_{J-1}\|_2 \le M_J$ for some constant $M_J$. Defining $\tilde{b}_\beta \coloneqq L + 2 M_J + \gamma$, we have $\|\nabla_\beta \Phi(\beta^{(J)})\| \le \tilde{b}_\beta\|\beta^{(J)} - \beta^{(J-1)}\|$. For $J \ge 1$, there exists a finite constant $\kappa_J > 0$ such that $\|\beta^{(J)} - \beta^{(J-1)}\| \le \kappa_J\big\|\sum_{j=0}^{J-1} \Delta\beta^{(j)}\big\| = \kappa_J\|\beta^{k+1} - \beta^k\| $. Setting $b_\beta \coloneqq \tilde{b}_\beta \kappa_J > 0$ proves Eq.~\eqref{eq:lm-approx-stat}.
\end{proof}

\begin{lemma}[Joint descent and gradient bound]
\label{lem:H1}
With the same requirements as for Proposition~\ref{prop:lm-descent}, if the upper-level step size satisfies $\eta_\omega < 2/L$, then there exists $c = \min\left(c_\beta, \tfrac{1}{\eta_\omega} - \tfrac{L}{2}\right) > 0$ and $b = \max\left( \sqrt{\left(\tfrac{1}{\eta_\omega} + L\right)^2 + 2L^2},b_\beta \sqrt{2} \right)$ such that 
\begin{align}
\Phi(z^{k+1}) &\le \Phi(z^k) - c\|z^{k+1} - z^k\|^2, \label{eq:H1} \\
\|\nabla \Phi(z^{k+1})\| &\le b\|z^{k+1} - z^k\|. \label{eq:H2}
\end{align}

\end{lemma}

\begin{proof}
Each iteration $k$ consists of two sub-steps, with $\beta$ update $(\omega^k, \beta^k) \to (\omega^k, \beta^{k+1})$ through inner LM iterations, followed by $\omega$ update $(\omega^k, \beta^{k+1}) \to (\omega^{k+1}, \beta^{k+1})$ through gradient descent. For $\beta$ update, Eq.~\eqref{eq:lm-suff-decrease} directly gives
\begin{equation}
\label{eq:H1-beta}
\Phi(\omega^k, \beta^{k+1})
\le \Phi(\omega^k, \beta^k) - c_\beta\|\beta^{k+1} - \beta^k\|^2.
\end{equation}

For $\omega$ update, the descent lemma for $L$-smooth functions gives,
\begin{equation*}
\Phi(\omega^{k+1}, \beta^{k+1})
\le
\Phi(\omega^k, \beta^{k+1})
+ \langle \nabla_\omega \Phi(\omega^k, \beta^{k+1}), \omega^{k+1} - \omega^k\rangle
+ \tfrac{L}{2}\|\omega^{k+1} - \omega^k\|^2.
\end{equation*}
Substituting the upper-level update rule $\omega^{k+1} - \omega^k = -\eta_\omega\nabla_\omega \Phi(\omega^k, \beta^{k+1})$, the inner-product term becomes
\begin{equation*}
\langle \nabla_\omega \Phi(\omega^k, \beta^{k+1}), \omega^{k+1} - \omega^k\rangle
= -\tfrac{1}{\eta_\omega}\|\omega^{k+1} - \omega^k\|^2.
\end{equation*}
Hence we have
\begin{equation}
\label{eq:H1-omega}
\Phi(\omega^{k+1}, \beta^{k+1})
\le \Phi(\omega^k, \beta^{k+1})
- \big(\tfrac{1}{\eta_\omega} - \tfrac{L}{2}\big)\|\omega^{k+1} - \omega^k\|^2.
\end{equation}
Define $c_\omega := \tfrac{1}{\eta_\omega} - \tfrac{L}{2}$, for a small enough learning rate $\eta_\omega < 2/L$ which ensures $c_\omega > 0$. Combining Eq.~\eqref{eq:H1-beta} and~\eqref{eq:H1-omega},
\[
\Phi(z^{k+1})
\le \Phi(z^k)
- c_\beta\|\beta^{k+1} - \beta^k\|^2
- c_\omega\|\omega^{k+1} - \omega^k\|^2.
\]
Setting $c := \min(c_\beta, c_\omega) > 0$ proves Eq.~\eqref{eq:H1}.

For the gradient bound, the update of $\omega$ and joint $L$-smoothness of Assumption~\ref{asm:bounded} give
\begin{equation}
\begin{aligned}
\label{eq:H2-omega-shift}
&\nabla_\omega \Phi(\omega^k, \beta^{k+1}) = -\tfrac{1}{\eta_\omega}(\omega^{k+1} - \omega^k), \\
&\|\nabla_\omega \Phi(\omega^{k+1}, \beta^{k+1}) - \nabla_\omega \Phi(\omega^k, \beta^{k+1})\|
\le L\|\omega^{k+1} - \omega^k\|.
\end{aligned}
\end{equation}
Combining both equations with the triangle inequality,
\begin{equation}
\label{eq:H2-omega-bound}
\|\nabla_\omega \Phi(z^{k+1})\|
\le \big(\tfrac{1}{\eta_\omega} + L\big)\|\omega^{k+1} - \omega^k\|.
\end{equation}

For the $\beta$ update, we decompose the gradient as
\begin{equation}
\begin{aligned}
\label{equ:H2-beta}
\|\nabla_\beta \Phi(\omega^{k+1}, \beta^{k+1})\|&=\|\big[\nabla_\beta \Phi(\omega^{k+1}, \beta^{k+1}) - \nabla_\beta \Phi(\omega^k, \beta^{k+1})\big]+\nabla_\beta \Phi(\omega^k, \beta^{k+1})\|,\\
& \le \|\nabla_\beta \Phi(\omega^{k+1}, \beta^{k+1}) - \nabla_\beta \Phi(\omega^k, \beta^{k+1}) \|+\|\nabla_\beta \Phi(\omega^k, \beta^{k+1})\|,\\
& \le L\|\omega^{k+1} - \omega^k\| + b_\beta\|\beta^{k+1} - \beta^k\|.
\end{aligned}
\end{equation}
Combining Eq.~\eqref{eq:H2-omega-bound} and~\eqref{equ:H2-beta},
\begin{align*}
\|\nabla\Phi(z^{k+1})\|^2
&= \|\nabla_\omega \Phi(z^{k+1})\|^2 + \|\nabla_\beta \Phi(z^{k+1})\|^2 \\
&\le \big(\tfrac{1}{\eta_\omega} + L\big)^2 \|\omega^{k+1} - \omega^k\|^2 + \big( L\|\omega^{k+1} - \omega^k\| + \ b_\beta\|\beta^{k+1} - \beta^k\|\big)^2, \\
&\le \big(\tfrac{1}{\eta_\omega} + L\big)^2 \|\omega^{k+1} - \omega^k\|^2
+ 2L^2 \|\omega^{k+1} - \omega^k\|^2
+ 2 b_\beta^2 \|\beta^{k+1} - \beta^k\|^2 \\ 
&= \Big(\big(\tfrac{1}{\eta_\omega} + L\big)^2 + 2L^2 \Big)\big(\|\omega^{k+1} - \omega^k\|^2\big) + 2 b_\beta^2\big(\|\beta^{k+1} - \beta^k\|^2\big).
\end{align*}
Setting $b = \max\left( \sqrt{\left(\tfrac{1}{\eta_\omega} + L\right)^2 + 2L^2},\ b_\beta\sqrt{2} \right)$ and taking square roots proves Eq.\eqref{eq:H2}.

\end{proof}

We now combine the joint sufficient decrease and gradient bound established in Lemma~\ref{lem:H1} with the boundedness assumption to derive the convergence properties of the FALM-PINN algorithm.

\begin{theorem}[Global convergence to critical points]
\label{thm:classical}
The iterate sequence $\{z^k\}_{k \ge 0}$ generated by Algorithm~\ref{alg:ifef} Phase~2 satisfies:
\begin{enumerate}
\item[\rm(i)] $\{\Phi(z^k)\}_{k \ge 0}$ is non-increasing and converges to some limit $\Phi^\star \ge 0$;
\item[\rm(ii)] $\sum_{k=0}^{\infty} \|z^{k+1} - z^k\|^2 \le \tfrac{1}{c}\bigl(\Phi(z^0) - \Phi^\star\bigr) < \infty$, and consequently $\|z^{k+1} - z^k\| \to 0$ as $k \to \infty$;
\item[\rm(iii)] $\|\nabla \Phi(z^k)\| \to 0$ as $k \to \infty$;
\item[\rm(iv)] $\{z^k\}$ admits at least one accumulation point, and every accumulation point $z^\star$ is a critical point of $\Phi$, i.e., $\nabla \Phi(z^\star) = 0$.
\end{enumerate}
\end{theorem}

\begin{proof}
The joint sufficient decrease in Lemma~\ref{lem:H1} gives $\Phi(z^{k+1}) \le \Phi(z^k) - c\|z^{k+1} - z^k\|^2 \le \Phi(z^k)$ for every $k \ge 0$, so the sequence $\{\Phi(z^k)\}$ is non-increasing. Since $\Phi(z)  \ge 0$, the sequence is also bounded below by zero. The monotone convergence theorem then yields a limit $\Phi^\star = \lim_{k \to \infty} \Phi(z^k) \in [0, \Phi(z^0)]$.

Rearranging the sufficient decrease inequality gives $c\|z^{k+1} - z^k\|^2 \le \Phi(z^k) - \Phi(z^{k+1})$, and summing over $k = 0, 1, \ldots, N$:
\[
c \sum_{k=0}^{N} \|z^{k+1} - z^k\|^2 \le \Phi(z^0) - \Phi(z^{N+1}) \le \Phi(z^0) - \Phi^\star.
\]
As $N\to \infty$, the right-hand side is a finite constant. This implies that the infinite series is bounded, and hence we have $\|z^{k+1} - z^k\| \to 0$.

The joint gradient bound in Lemma~\ref{lem:H1} gives $\|\nabla \Phi(z^{k+1})\| \le b \|z^{k+1} - z^k\|$. Then the right-hand side converges to zero by (ii) as $k\to \infty$, thus $\|\nabla \Phi(z^k)\| \to 0$.

For the last claim, since the sequence $\{z^k\} \subset \mathcal{B}$ with $\mathcal{B}$ compact from Assumption~\ref{asm:bounded}, there exists a convergent subsequence $z^{k_j} \to z^\star \in \mathcal{B}$. The joint $L$-smoothness in Assumption~\ref{asm:bounded} ensures that $\nabla \Phi$ is Lipschitz and therefore continuous, hence $\nabla \Phi(z^{k_j}) \to \nabla \Phi(z^\star)$ as $j \to \infty$. By (iii), $\|\nabla \Phi(z^{k_j})\| \to 0$, and combining these two limits yields $\nabla \Phi(z^\star) = 0$.
\end{proof}

\section{Numerical experiments}
\label{sec:experiments}
In this section, we conduct experiments to evaluate the proposed FALM-PINN framework against representative and state-of-the-art baselines across a range of PDE benchmarks, together with a plain PINN (vanilla PINN) as a reference. Each baseline adopts a distinct strategy for mitigating spectral bias:

\begin{itemize}
    \item \textbf{Residual-Based Attention (RBA)}~\cite{anagnostopoulos2024residual}: 
    RBA is a gradient-free, point-wise reweighting strategy by calculating the moving average of normalized residuals, guiding the network to adaptively focus on high-error regions.

    \item \textbf{Physics-Informed Kolmogorov--Arnold Network 
    (PIKAN)}~\cite{wang2025kolmogorov}: PIKAN leverages the Kolmogorov--Arnold Network architecture with learnable B-spline functions to replace traditional MLPs, effectively mitigating spectral bias on high-frequency and multi-scale PDEs.

    \item \textbf{CompleX-PINN}~\cite{si2025complexphysicsinformedneuralnetwork}:
    CompleX-PINN employs a learnable Cauchy activation function within a single hidden layer, reducing parameter complexity compared to conventional 
    deep MLPs while targeting stiff and high-frequency PDEs.

    \item \textbf{Sinusoidal Representation Networks 
    (SIREN)}~\cite{sitzmann2020implicit}: SIREN replaces the standard smooth activation functions with periodic sinusoidal activations $\phi_i(x) = \sin(\omega_0W_ix+b_i)$, equipping the network with an inherent inductive bias toward high-frequency functions. This enables the network to represent fine-scale solution features that are typically difficult to learn under spectral bias.
\end{itemize}

\paragraph{Training and testing points} The testing dataset consists of uniformly 
distributed grid points: for 2D spatial or 1D spatio-temporal PDEs, 
we use a uniform $300 \times 300$ grid of $90{,}000$ points, while 
for 2D spatio-temporal PDEs, we use a uniform $100 \times 100 \times 100$ 
grid of $100{,}000$ points. Training samples are drawn independently 
via Latin Hypercube Sampling (LHS) from the relevant domains, where 
$N_f$, $N_b$, and $N_i$ denote the number of physics collocation, 
boundary, and initial condition points, respectively. The corresponding 
loss weights for these three terms are set to $0.01$, $1$, and 
$1$ uniformly across all experiments. The testing and training sets are not necessarily disjoint, as both are independently sampled from the same domain.

Each method is trained and evaluated over five independent trials with different random seeds.
Unless specified, methods that employ a standard MLP use a fully connected network with $\tanh$ activations; methods with specialized architectures (PIKAN, compleX-PINN, SIREN) follow their respective configurations.
All methods are optimized using Adam, and we follow the hyperparameter settings or adopt the 
tuning guidelines recommended for all baselines in their original publications.
Model accuracy is assessed using the relative $L^2$ error, defined as:
\begin{equation}
\label{eq:rel_l2}
   \text{Relative } L^2 \text{ error} = \frac{\sqrt{\sum\nolimits_{k=1}^{N} 
    \bigl|\hat{u}(\mathbf{x}_k) - u(\mathbf{x}_k)\bigr|^2}}
    {\sqrt{\sum\nolimits_{k=1}^{N} 
    \bigl|u(\mathbf{x}_k)\bigr|^2}},
\end{equation}
where $u$ denotes the exact PDE solution, $\hat{u}$ is the predicted 
output of the model, and $N$ is the number of points in the 
test set. We report the mean $\pm$ standard deviation across 
all five independent trials in the numerical tables, and present the heatmap 
results corresponding to the best-performing trial for each method. 
All experiments are conducted on a single NVIDIA RTX 4090 GPU.

\subsection{2D Klein--Gordon equation}
The Klein--Gordon equation serves as a representative test of long-time approximation in nonlinear wave propagation. We consider the following initial-boundary value problem on the spatio-temporal domain $\Omega = T\times \Omega_s $ with $T = [0, 10]$ and $\Omega_s=[0,1]^2$:
\begin{equation}
\begin{aligned}
    &u_{tt} - \Delta u + u^2 = f, \quad (t,x,y) \in \Omega,\\
    &u(0,x,y) = x + y, \quad (x,y) \in \Omega_s,\\
    &u_t(0,x,y) = xy, \quad (x,y) \in \Omega_s,\\
    &u(t,x,y) = g(t,x,y), \quad (x,y) \in \partial\Omega_s,\ t \in T,
\end{aligned}
\end{equation}
where the source term $f$ and boundary condition $g$ are obtained from the following analytical solution:
\begin{equation}
    u(t, x, y) = (x + y)\cos(t) + xy\sin(t).
\end{equation}

All MLP-based methods use a fully connected network with $4$ hidden layers of width $80$, and are trained for $50{,}000$ iterations with $N_f = 1{,}000$ collocation points, $N_b = 300$ boundary points, and $N_i = 100$ initial points. The baselines use Adam with an initial learning rate $5\times10^{-3}$, 
decayed by a factor of $0.8$ every $1{,}000$ iterations to a 
minimum of $1\times10^{-5}$.

Following Algorithm~\ref{alg:ifef}, training is split into two 
phases. In the warm-up phase ($N_{\mathrm{warm}} = 5{,}000$ 
iterations), both $\omega$ and $\beta$ are jointly optimized by Adam 
with an initial learning rate of $5\times10^{-3}$, decayed by a 
factor of $0.8$ every $1{,}000$ iterations. The Fourier feature mapping uses $D = 800$ features with 
bandwidth $\sigma = 1$.
In the subsequent alternating optimization phase ($N_{\mathrm{alter}} = 45{,}000$ 
iterations), the upper-level parameters~$\omega$ are updated by 
Adam with a reduced learning rate of $1\times10^{-4}$, following the same 
decay schedule to a floor of $1\times10^{-5}$. Each lower-level solve performs $J = 2$ damped subproblems with a fixed damping parameter $\gamma = 1\times10^{-6}$.

The method-specific hyperparameters for each baseline are as follows.
The RBA approach~\cite{anagnostopoulos2024residual} applies residual-based attention 
reweighting with parameters $\eta_{\mathrm{RBA}}=0.001$ and 
$\gamma_{\mathrm{RBA}} = 0.999$.
CompleX-PINN~\cite{si2025complexphysicsinformedneuralnetwork} employs a single Cauchy layer of width $200$.
PIKAN~\cite{wang2025kolmogorov} uses a 
Kolmogorov--Arnold Network of hidden structure $[3, 8, 8, 8, 1]$, 
grid size $10$, and spline order $3$.
SIREN~\cite{sitzmann2020implicit} employs sinusoidal activations following the original implementation, with $\omega_0 = 30$.

As illustrated in Fig.~\ref{fig:kg2d_convergence}, FALM-PINN exhibits a distinct two-stage behavior: during the initial warm-up phase, its convergence rate remains comparable to that of the baselines. After switching to the alternating optimization phase, the error decreases by nearly two orders of magnitude while maintaining stable convergence throughout training. This sharp improvement indicates that the warm-up stage provides a suitable feature basis and coefficient initialization, while the lower-level subproblem updates in Eq.~\eqref{equ:lm_subproblem} rapidly refine the projection coefficients to produce a more accurate approximation. This demonstrates the advantage of decoupling basis learning from coefficient optimization for nonlinear PDEs. The final average relative $L^2$ errors, reported in Table~\ref{tab:kg2d}, show that FALM-PINN achieves errors two orders of magnitude lower than all baselines. Meanwhile, its training time and peak memory remain within the range spanned by the baselines, indicating that this substantial accuracy gain is not achieved at the expense of excessive computational cost.

 \begin{table}[htbp]
\centering
\caption{2D Klein--Gordon equation: relative $L^2$ errors (mean $\pm$ std over five independent runs), GPU time per $1000$ iterations, and peak GPU memory for all methods. The lowest value in each column is shown in bold.}
\label{tab:kg2d}
\setlength{\tabcolsep}{6pt}
\renewcommand{\arraystretch}{1.15}
\begin{tabular}{lccc}
\toprule
Method & Relative $L^2$ error & Time (s) & Memory (GB) \\
\midrule
Vanilla PINN  & $8.11\times10^{-3}\pm3.98\times10^{-4}$ & $\mathbf{4.06}$  & $\mathbf{0.07}$\\
RBA           & $2.96\times10^{-3}\pm1.05\times10^{-3}$ & $4.07$  & $0.64$\\
compleX-PINN  & $4.03\times10^{-3}\pm4.94\times10^{-4}$ & $5.05$  & $3.77$ \\
PIKAN         & $1.14\times10^{-3}\pm1.68\times10^{-4}$ & $36.03$ & $1.95$ \\
SIREN         & $4.02\times10^{-3}\pm7.79\times10^{-4}$ & $5.81$  & $0.95$\\
\midrule
\textbf{FALM-PINN} & $\mathbf{3.50\times10^{-5}\pm4.88\times10^{-6}}$ & $7.94$ & $0.25$\\
\bottomrule
\end{tabular}
\end{table}

\begin{figure}[htbp]
    \centering
    \includegraphics[width=0.99\linewidth]{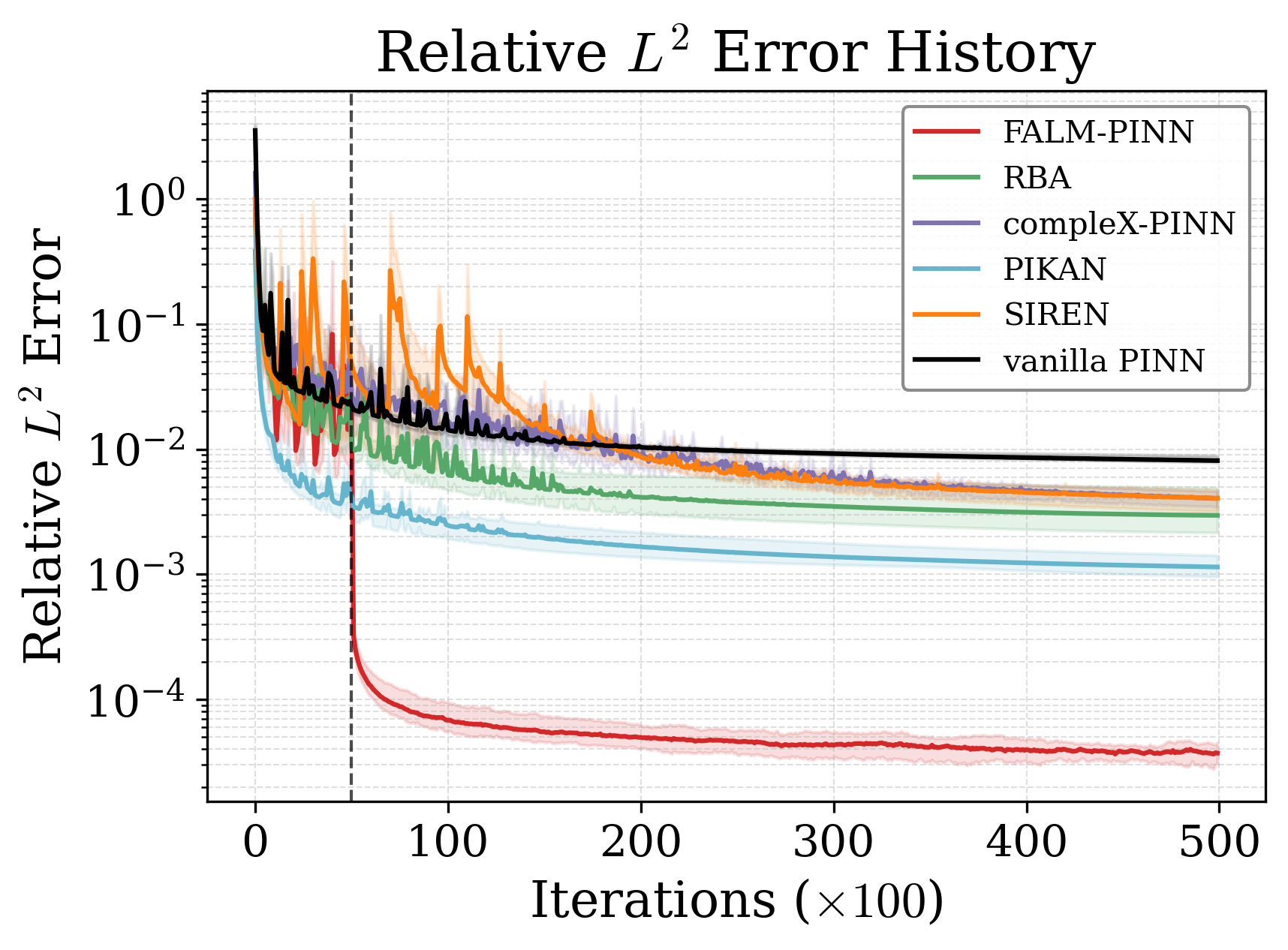}
    \caption{Convergence trajectories on the 2D Klein--Gordon equation. Relative $L^2$ error against training iterations; curves denote the mean across five independent trials, and shaded bands indicate the min--max range. The vertical dashed line marks the transition from the warm-up phase ($N_{\mathrm{warm}}=5{,}000$) to the alternating optimization phase.}
    \label{fig:kg2d_convergence}
\end{figure}
To further assess the approximation, Fig.~\ref{fig:kg2d_exact} and~\ref{fig:kg2d_heatmap} display the exact solution and point-wise absolute error maps for all methods at three time slices $t \in \{2.5, 5.0, 7.5\}$. FALM-PINN yields consistently low errors throughout the spatial domain and over long-time evolution, with the maximum point-wise absolute error remaining below $3\times10^{-5}$. These results validate the ability of FALM-PINN to accurately resolve nonlinear PDEs over long time horizons with a small training set.

\begin{figure}[htbp]
    \centering
    \includegraphics[width=0.99\linewidth]{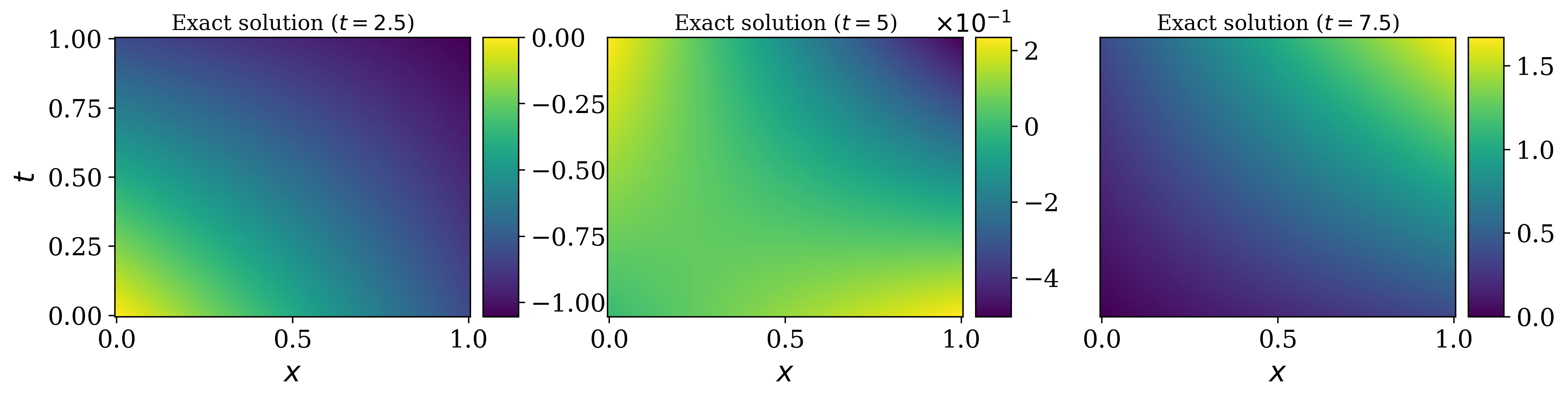}
    \caption{Exact solution of the 2D Klein--Gordon equation at three time slices $t \in \{2.5, 5.0, 7.5\}$.}
    \label{fig:kg2d_exact}
\end{figure}

\begin{figure}[p]
    \centering
    \includegraphics[width=0.99\linewidth, height=0.95\textheight, keepaspectratio]{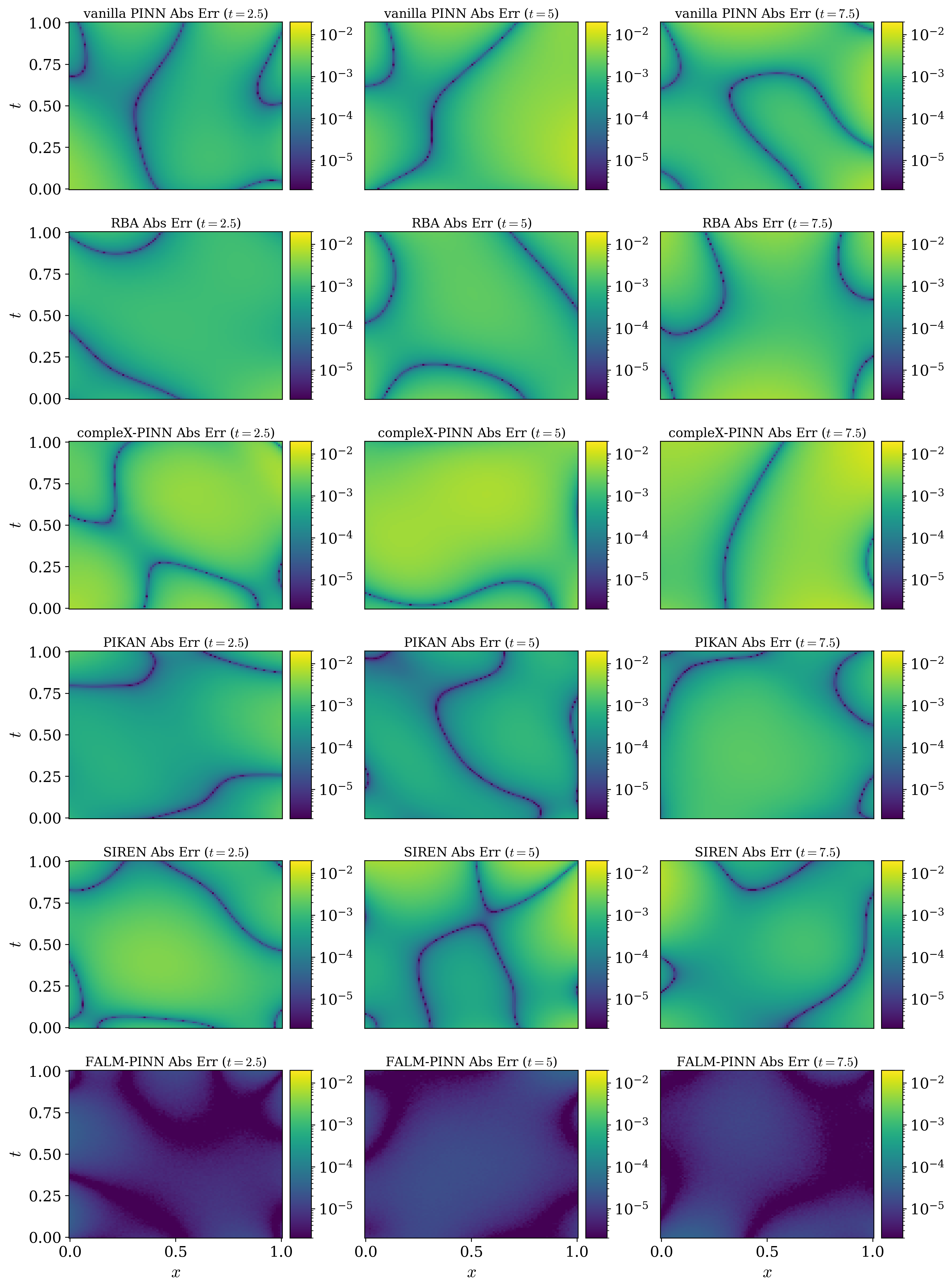}
    \caption{Point-wise absolute error maps in log scale on the 2D Klein--Gordon equation at time slices $t \in \{2.5, 5.0, 7.5\}$ (columns) for vanilla PINN, RBA, compleX-PINN, PIKAN, SIREN, and FALM-PINN respectively.}
    \label{fig:kg2d_heatmap}
\end{figure}

\subsection{1D Korteweg--de Vries equation}
\label{sec:kdv1d}
The Korteweg--de Vries (KdV) equation couples nonlinear convection with third-order dispersion, producing dispersive wave dynamics that pose a challenge for PINNs. Following~\cite{penwarden2023unified}, we consider the equation over the short time horizon on the spatio-temporal domain $\Omega = T\times \Omega_s $ with $\Omega_s = [-1, 1]$ and $T = [0, 1]$. 
\begin{equation}
\begin{aligned}
    &u_t + u u_x + 0.0025\, u_{xxx} = 0, \quad (t, x) \in \Omega,\\
    &u(0, x) = \cos(\pi x), \quad x \in \Omega_s,\\
    &u(t, -1) = u(t, 1),\ \ u_x(t, -1) = u_x(t, 1),\ \ u_{xx}(t, -1) = u_{xx}(t, 1), \quad t \in T.
\end{aligned}
\end{equation}
The reference solution is shown in Fig.~\ref{fig:kdv1d_reference}.

\begin{figure}
    \centering
    \includegraphics[width=0.5\linewidth]{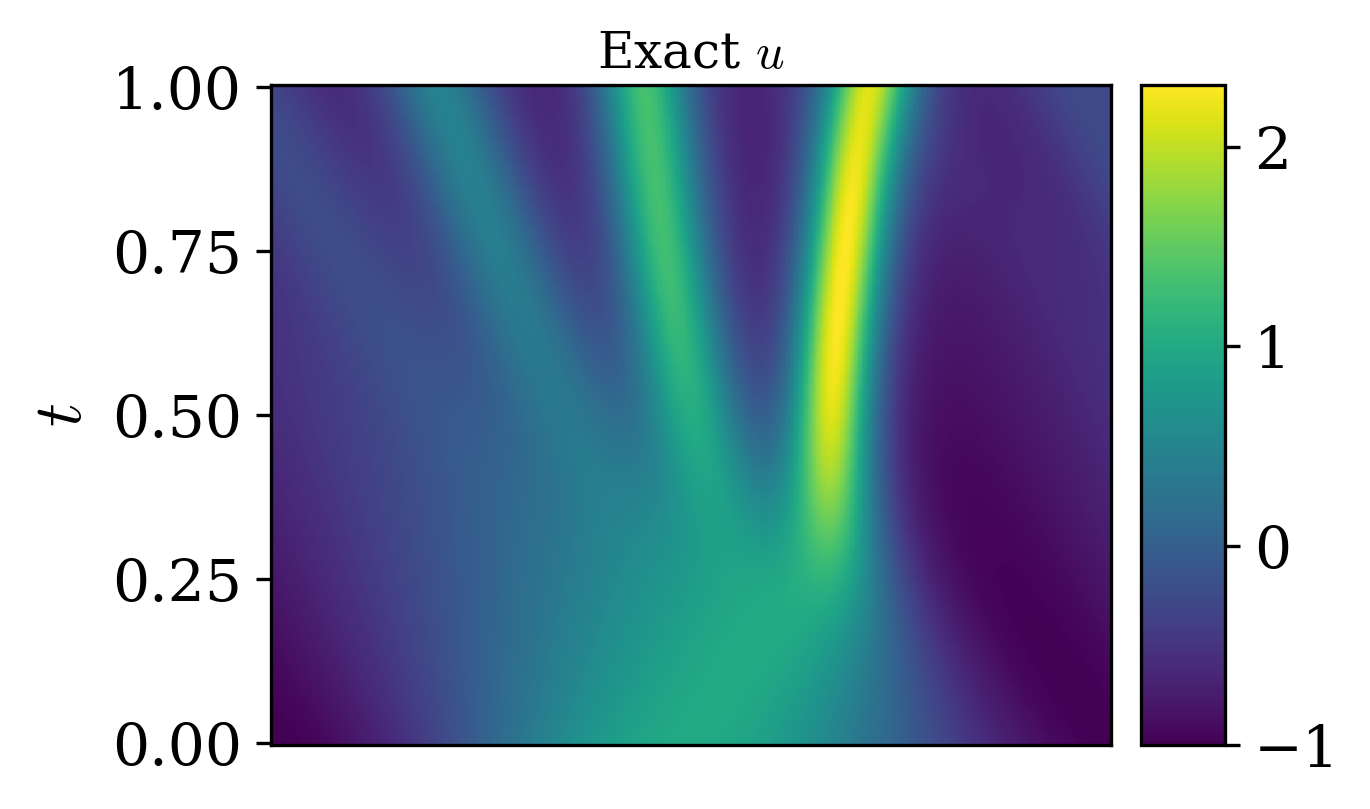}
    \caption{Reference solution of the 1D Korteweg--de Vries equation.}
    \label{fig:kdv1d_reference}
\end{figure}

All MLP-based methods share the same architecture of $6$ hidden layers of width $64$. We use $N_f = 10{,}000$ collocation points, $N_i = N_b = 200$ initial and boundary points. The baselines are optimized by Adam with an initial learning rate of $2 \times 10^{-3}$, exponentially decayed by a factor of $0.93$ every $2{,}000$ iterations down to a floor of $1 \times 10^{-5}$, and trained for $60{,}000$ iterations.

During the warm-up phase of FALM-PINN ($N_{\mathrm{warm}} = 5{,}000$), all parameters are jointly trained by Adam at a fixed learning rate of $2 \times 10^{-3}$. The Fourier feature mapping uses $D = 800$ features and bandwidth $\sigma = 1$. The subsequent alternating optimization phase ($N_{\mathrm{alter}} = 55{,}000$) updates the upper-level parameters $\omega$, starting from a learning rate of $1 \times 10^{-4}$ and following the same exponential decay scheduler as the baselines. Each lower-level solve performs $J = 2$ LM iterations with the damping initialized at $\gamma = 1\times 10^{-5}$ and decayed by a factor of $0.5$ every $2{,}000$ iterations, until below $1\times10^{-6}$.

All baseline configurations are as follows. RBA~\cite{anagnostopoulos2024residual} uses $\eta_{\mathrm{RBA}} = 0.001$ and $\gamma_{\mathrm{RBA}} = 0.999$. CompleX-PINN~\cite{si2025complexphysicsinformedneuralnetwork} adopts a single Cauchy layer of width $200$. PIKAN~\cite{wang2025kolmogorov} uses a Kolmogorov--Arnold network of structure $[2, 8, 8, 8, 1]$ with grid size $10$ and spline order $3$. SIREN~\cite{sitzmann2020implicit} follows its default
setting with $\omega_0=30$.

Table~\ref{tab:kdv1d} and Fig.~\ref{fig:kdv1d_convergence} summarize the numerical results for all methods. During the joint warm-up phase, FALM-PINN exhibits a convergence trend comparable to that of the baseline PINNs. After switching to the alternating phase, its relative $L^2$ error drops sharply and subsequently converges to
$4.42\times10^{-4}$. This behavior indicates that the learned Fourier-enhanced basis provides an effective
representation, while the lower-level LM updates rapidly refine the projection coefficients. The results support the effectiveness of separating basis learning from coefficient fitting for nonlinear dispersive dynamics.

\begin{table}[htbp]
\centering
\caption{Relative $L^2$ errors on the 1D Korteweg--de Vries equation
(mean $\pm$ standard deviation over five independent runs) for all methods. The lowest
error is shown in bold.}
\label{tab:kdv1d}
\setlength{\tabcolsep}{6pt}
\renewcommand{\arraystretch}{1.15}
\begin{tabular}{lc}
\toprule
Method & Relative $L^2$ error  \\
\midrule
Vanilla PINN  & $4.61\times10^{-2} \pm 4.03\times10^{-3}$ \\
RBA           & $7.38\times10^{-3} \pm 1.32\times10^{-3}$ \\
compleX-PINN  & $9.63\times10^{-3} \pm 1.54\times10^{-3}$ \\
PIKAN         & $2.52\times10^{-2} \pm 7.31\times10^{-3}$ \\
SIREN         & $2.92\times10^{-2} \pm 6.18\times10^{-3}$ \\
\midrule
\textbf{FALM-PINN} & $\mathbf{4.42\times10^{-4} \pm 1.81\times10^{-4}}$ \\
\bottomrule
\end{tabular}
\end{table}

\begin{figure}
    \centering
    \includegraphics[width=0.99\linewidth]{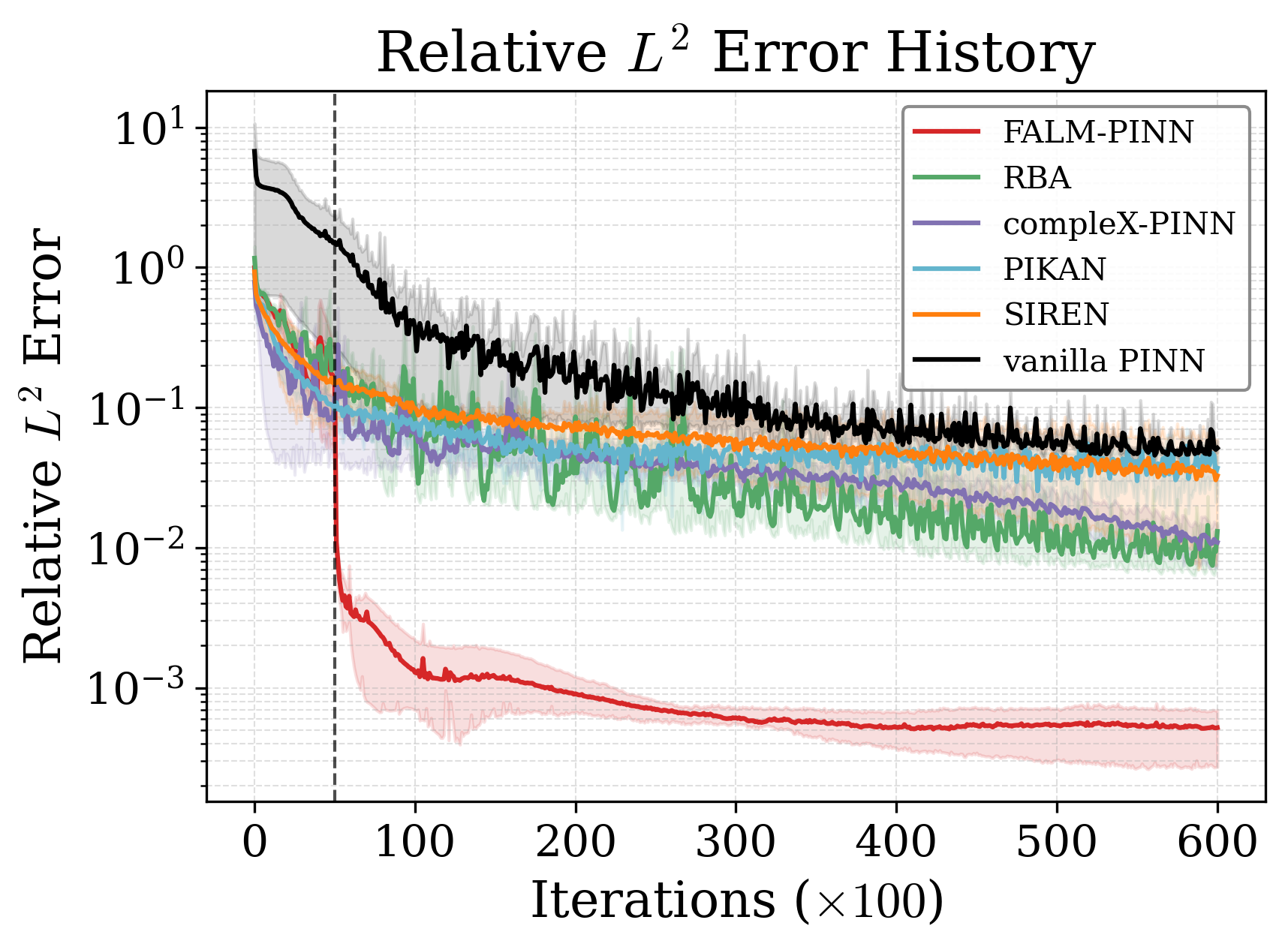}
    \caption{Relative $L^2$ error trajectories on the 1D Korteweg--de Vries equation. Curves denote the mean across five independent trials, and shaded bands indicate the min--max range. The vertical dashed line marks the transition from the warm-up phase ($N_{\mathrm{warm}}=5{,}000$) to the alternating optimization phase.}
    \label{fig:kdv1d_convergence}
\end{figure}

The predicted solutions and corresponding point-wise absolute errors are shown in Fig.~\ref{fig:kdv1d_heatmap}. FALM-PINN maintains uniformly lower errors across the spatio-temporal domain and more accurately resolves the narrow oscillatory wave patterns. These results demonstrate that the proposed method is effective in capturing the nonlinear dispersion structures.

\begin{figure}[htbp]
    \centering
    \includegraphics[width=0.99\linewidth, height=0.95\textheight, keepaspectratio]{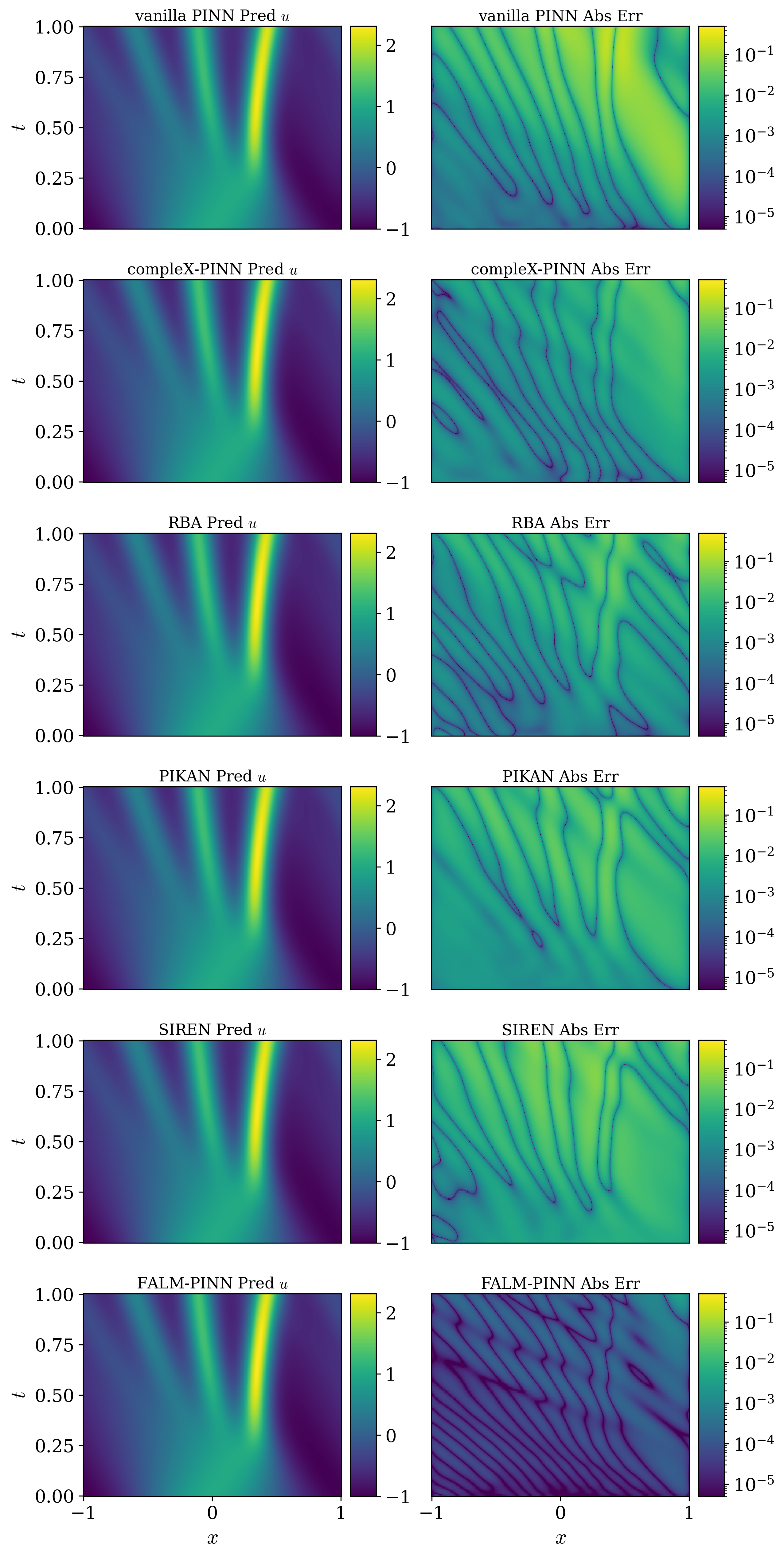}
    \caption{Predictions (left column) and absolute error maps in log scale (right column) on the 1D Korteweg--de Vries equation for vanilla PINN, RBA, compleX-PINN, PIKAN, SIREN, and FALM-PINN respectively.}
    \label{fig:kdv1d_heatmap}
\end{figure}

\subsection{1D Heat equation with high-Frequency solution}
High-frequency solutions are common failure modes for PINNs due to spectral bias. To evaluate the ability of FALM-PINN to capture oscillatory modes, we adopt the high-frequency heat equation benchmark from~\cite{wang2025kolmogorov}. The equation is defined on the spatio-temporal domain 
$\Omega = T\times\Omega_s$ with $T = [0,1]$ and $\Omega_s = [0,1] $:
\begin{equation}
\label{equ:heat_pde}
\begin{aligned}
    &u_t = \frac{1}{(F\pi)^2}u_{xx}, 
      \quad  (t,x) \in \Omega,\\
    &u(t,0) = u(t,1) = 0, 
      \quad t\in T, \\
      &u(0,x) = \sin(F\pi x), 
      \quad x\in \Omega_s,
\end{aligned}
\end{equation}
where $F$ denotes the spatial frequency. The analytical solution is $u(t,x) = e^{-t}\sin(F\pi x)$, which combines a slowly decaying temporal component with a highly oscillatory spatial mode. We set $F = 100$, which requires the network to recover high-frequency structures and is known to be challenging for standard PINNs. The exact solution is shown in Fig.~\ref{fig:heat1d_exact}.

\begin{figure}[htbp]
    \centering
    \includegraphics[width=0.5\linewidth]{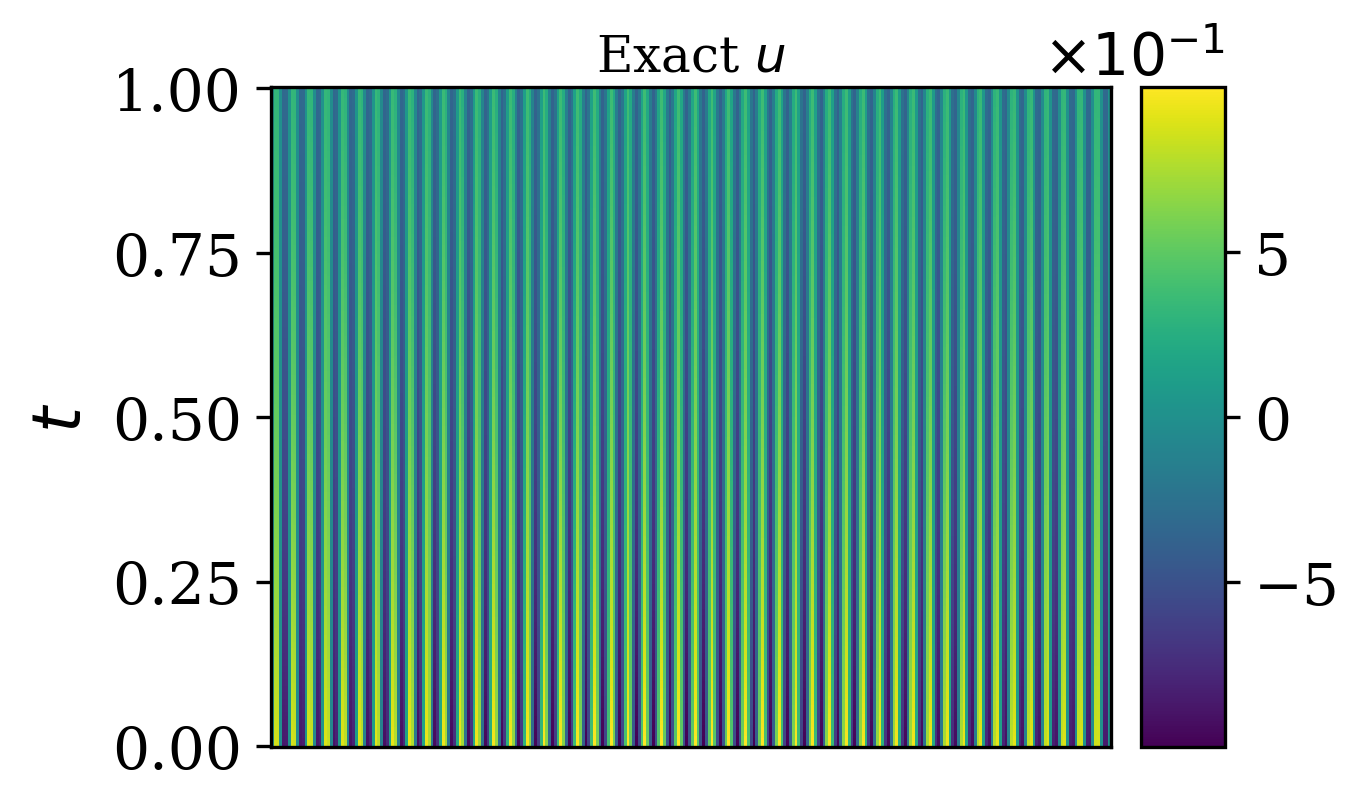}
    \caption{Exact solution of the 1D Heat equation.}
    \label{fig:heat1d_exact}
\end{figure}

All MLP-based methods use a fully connected network with $5$ hidden layers of width $80$. Training proceeds for $50{,}000$ iterations under Adam, with an initial learning rate of $1.5\times10^{-3}$ exponentially decayed by a factor of $0.8$ every $2{,}000$ iterations down to $1\times10^{-5}$. The training set consists of $N_f = 15{,}000$ collocation points, $N_b = 200$ boundary points, and $N_i = 200$ initial points.

Since the heat equation is linear, the lower-level problem of FALM-PINN is convex in $\beta$ and admits a closed-form one-step update, so no warm-up phase is required
($N_{\mathrm{warm}} = 0$). The Fourier feature mapping uses $D = 800$ features with bandwidth $\sigma = 10$, and the damping is set to $\gamma = 1\times10^{-7}$.
 
For the baselines, we follow the high-frequency tuning recommendations reported in their original papers.
RBA retains $\eta_{\mathrm{RBA}} = 0.001$ and 
$\gamma_{\mathrm{RBA}} = 0.999$. CompleX-PINN widens its single Cauchy layer to $500$, while PIKAN employs a $[2, 5, 5, 1]$ Kolmogorov–Arnold network with grid size $150$ and spline order~$3$.
SIREN~\cite{sitzmann2020implicit} employs sinusoidal activations with $\omega_0 = 100$, increased from the default value of $30$ to better capture the high-frequency components present in the target solution.

\begin{figure}[htbp]
    \centering
    \includegraphics[width=\linewidth]{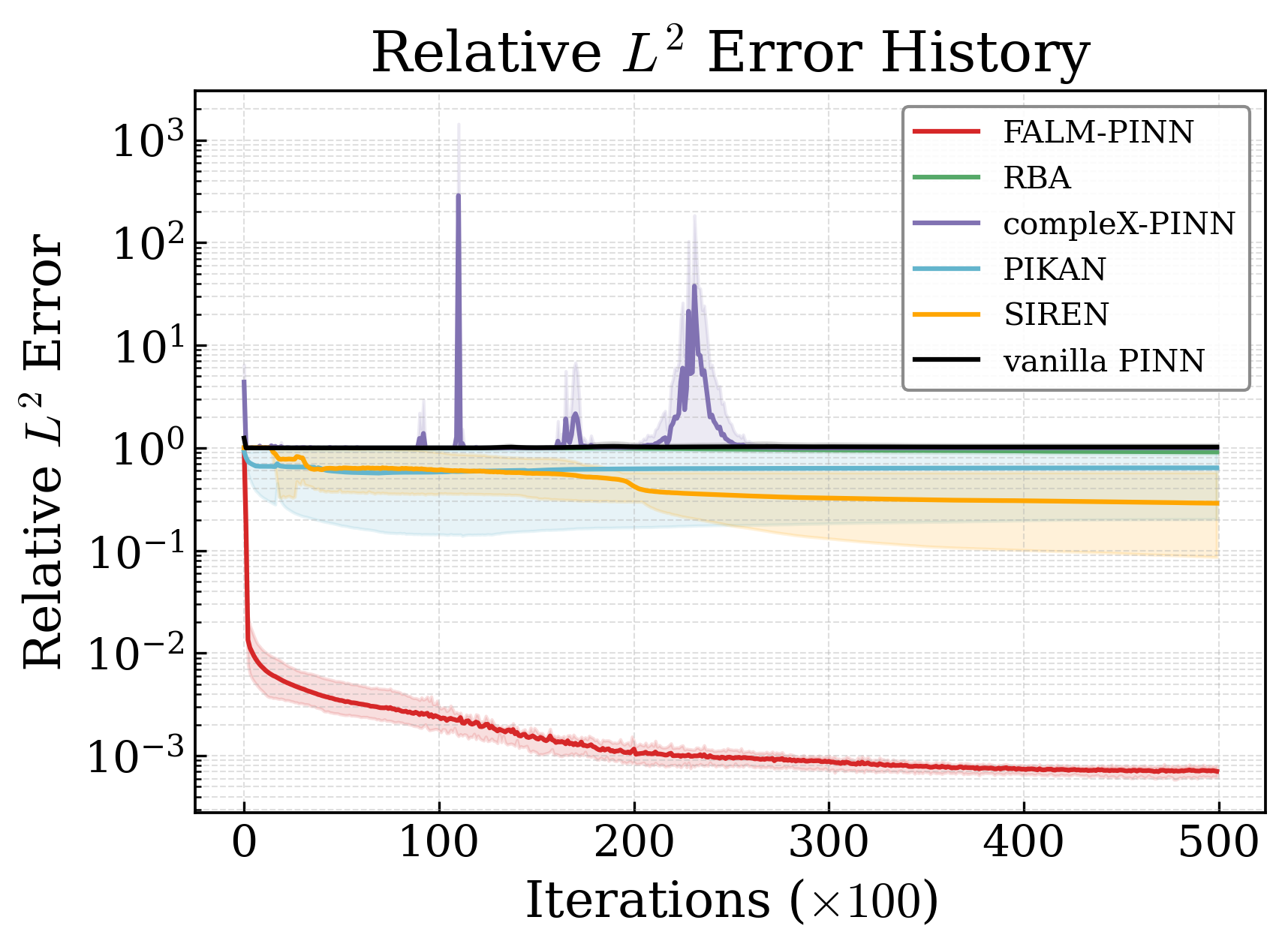}
    \caption{Convergence trajectories on the 1D heat equation. Relative $L^2$ error against training iterations; curves denote the mean across five independent trials, and shaded bands indicate the min--max range.}
    \label{fig:heat1d_convergence}
\end{figure}
The training histories for all methods are shown in Fig.~\ref{fig:heat1d_convergence}. None of the baselines resolves the high-frequency mode within the $50{,}000$ iterations. Vanilla PINN, RBA and compleX-PINN fail in this approximation. PIKAN and SIREN attain lower errors but still plateau at a high floor of about $10^{-1}$. In contrast, FALM-PINN converges stably to an error of $6.8\times10^{-4}$, more than two orders of magnitude below the best baseline.

\begin{figure}[p]
    \centering
    \includegraphics[width=0.99\linewidth, height=0.95\textheight, keepaspectratio]{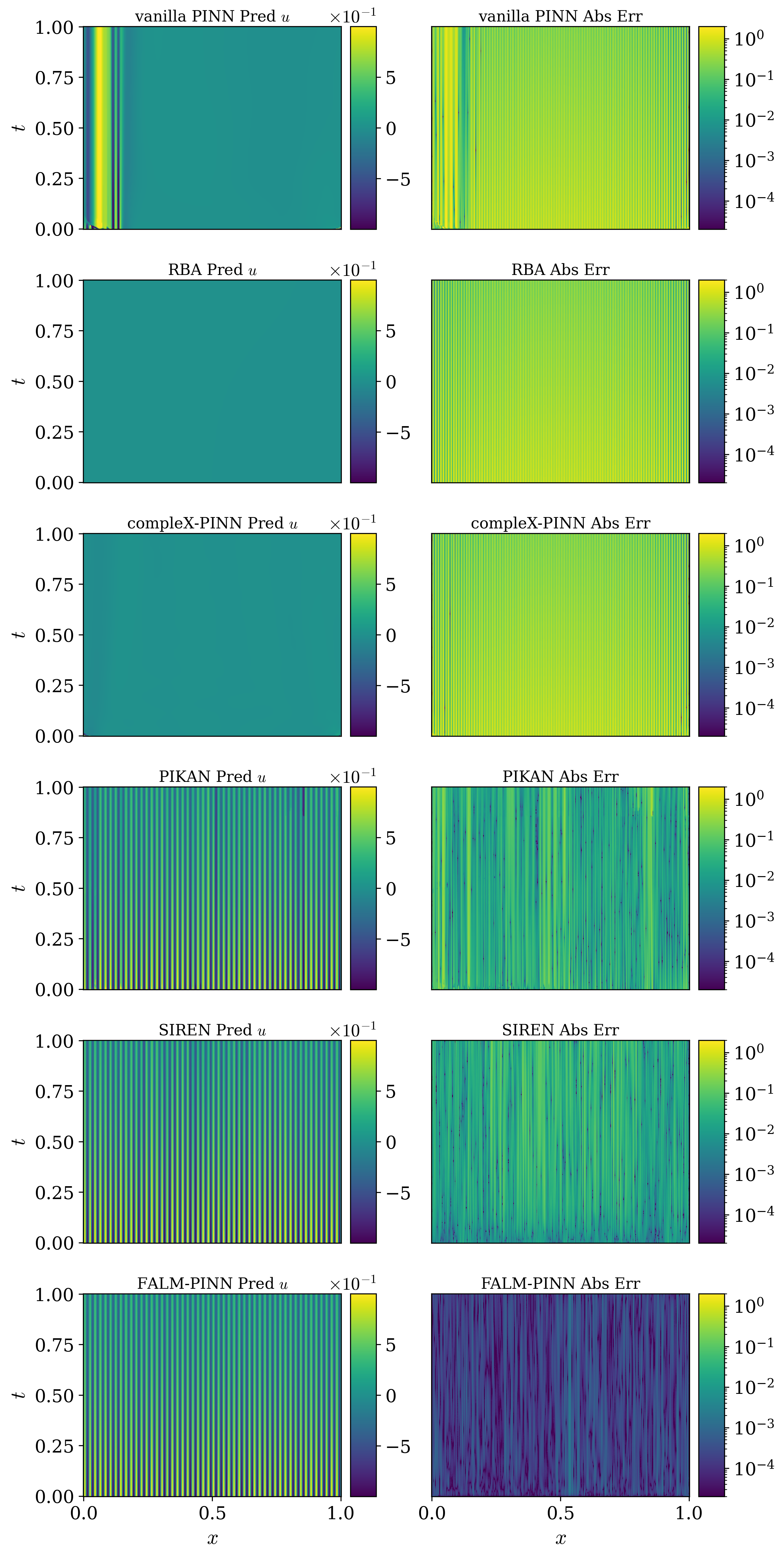}
    \caption{Predicted solutions (left column) and point-wise absolute errors in log scale (right column) on the 1D heat equation for vanilla PINN, RBA, compleX-PINN, PIKAN, SIREN, and FALM-PINN, respectively.}
    \label{fig:heat1D_pointwise}
\end{figure}
As demonstrated in Fig.~\ref{fig:heat1D_pointwise}, vanilla PINN, RBA, and compleX-PINN collapse to trivial, near-zero predictions that miss the oscillation entirely. PIKAN and SIREN show intermediate performance, capturing the general oscillatory pattern but with substantial point-wise errors across the domain. FALM-PINN is the only method accurately recovering the high-frequency mode with uniformly small error. This consistent superiority confirms that FALM-PINN can effectively mitigate spectral bias.

To further assess the role of the Fourier-enhanced basis in overcoming spectral bias, we conduct an ablation on the bandwidth hyperparameter $\sigma$, which governs the frequency distribution of the induced kernel following Lemma~\ref{lem:kernel_convergence}. We vary the spatial frequency $F \in \{10, 100, 150, 200\}$ against the bandwidth $\sigma \in \{1, 10, 25, 50\}$. To ensure that the high-frequency modes at $F = 200$ are sufficiently sampled to satisfy the Nyquist-Shannon criterion~\citep{shannon2006communication, wu2025iterative}, we double the number of boundary and initial points to $N_b = 400$ and $N_i = 400$. All remaining parameters are kept identical to the main experiment. 

The results of the ablation study are summarized in Table~\ref{tab:heat1d_ablation_sigma}; and Fig.~\ref{fig:heat1d_ablation_slices} further visualizes the prediction slices at the terminal time $t=1$ across different spatial frequencies. Both demonstrate that the bandwidth required for accurate approximation increases with the spatial frequency $F$. For the low-frequency case with $F=10$, a narrow bandwidth $\sigma = 1$ already suffices to span the spectral domain of the solution, and all tested bandwidths accurately reproduce the exact solution. However, as $F$ grows, only a sufficiently large $\sigma$ introduces enough high-frequency components into the Fourier-enhanced features to capture the rapid oscillatory structure. In the extreme regime $F = 200$, $\sigma = 1$ and $\sigma = 10$ fail to recover meaningful predictions; $\sigma = 25$ partially recovers the oscillation pattern but still exhibits noticeable approximation error; only $\sigma = 50$ preserves an accurate approximation, achieving a relative $L^2$ error of $3.3 \times 10^{-3}$.

This trend is consistent with Lemma~\ref{lem:kernel_convergence}, which indicates that $\sigma$ should scale with the dominant frequency of the target solution so that the Fourier-enhanced features span the relevant spectral band. Notably, even in the extremely high-frequency regime with $F = 200$, a properly chosen $\sigma$ enables FALM-PINN to attain an accurate approximation, underscoring the role of the Fourier-enhanced basis in mitigating spectral bias.

\begin{table}[htbp]
\centering
\caption{Ablation of the Fourier-feature bandwidth $\sigma$ across spatial frequencies $F$ on the 1D heat equation. Each entry reports the
relative $L^2$ error (mean $\pm$ standard deviation); the lowest error in each row is shown in bold. A dash '$-$' denotes a configuration that failed to produce a meaningful prediction.}
\label{tab:heat1d_ablation_sigma}
\setlength{\tabcolsep}{10pt}
\renewcommand{\arraystretch}{1.15}
\resizebox{\linewidth}{!}{%
\begin{tabular}{c cccc}
\toprule
 & \multicolumn{4}{c}{Bandwidth $\sigma$} \\
\cmidrule(lr){2-5}
Frequency $F$ & $1$ & $10$ & $25$ & $50$ \\
\midrule
$10$  & $ 8.0\times 10^{-5} \pm 9.9\times 10^{-6}$ & $\mathbf{  3.6\times 10^{-5} \pm 7.5\times 10^{-7} }$ & $4.7\times 10^{-5} \pm 2.3\times 10^{-5}  $ & $ 5.0\times 10^{-5} \pm 5.0\times 10^{-6} $ \\
$100$ & $ - $ & $\mathbf{ 6.3\times 10^{-4}\pm9.5\times 10^{-5} }$ & $ 1.3\times 10^{-3}\pm 3.5\times 10^{-4}$ & $ 1.0\times 10^{-3}\pm1.3\times 10^{-5} $ \\
$150$ & $ - $ & $ 1.9\times 10^{-1}\pm1.3\times 10^{-1} $ & $ 2.5\times 10^{-3}\pm1.1\times 10^{-3} $ & $\mathbf{ 1.8\times 10^{-3}\pm3.4\times 10^{-4} }$ \\
$200$ & $ - $ & $ - $ & $ 5.9\times 10^{-1}\pm 8.1\times 10^{-3}$ & $\mathbf{ 3.3\times 10^{-3}\pm2.4\times 10^{-4}}$ \\
\bottomrule
\end{tabular}
}%
\end{table}

\begin{figure}[htbp]
    \centering
    \includegraphics[width=0.99\linewidth]{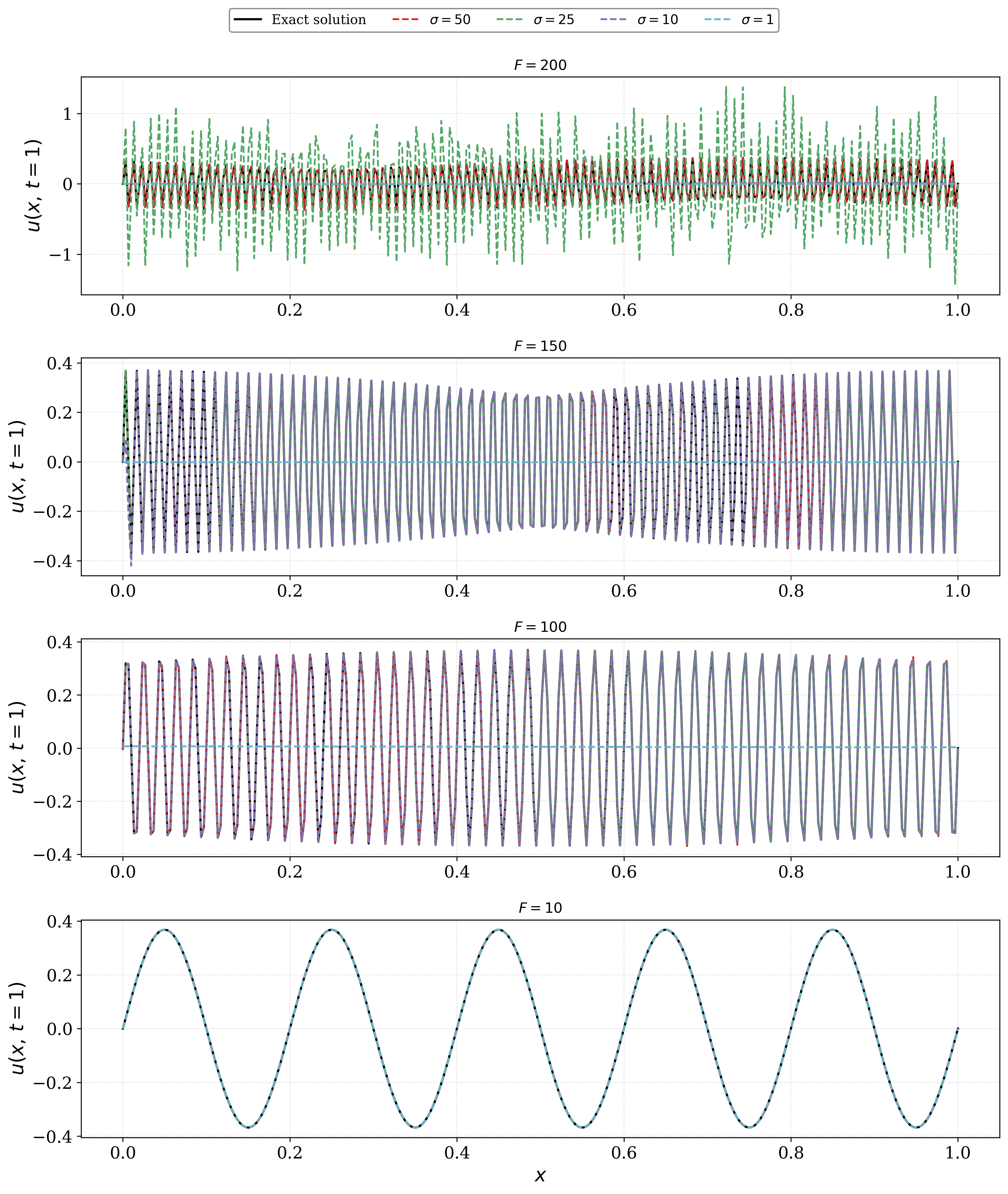}
    \caption{Prediction slices $u(x, t = 1)$ of FALM-PINN on the 1D heat equation for Fourier-feature bandwidths $\sigma \in \{1, 10, 25, 50\}$, at spatial frequencies $F \in \{200, 150, 100, 10\}$ (top to bottom). The black curve is the exact solution.}
    \label{fig:heat1d_ablation_slices}
\end{figure}

\subsection{Lid-driven cavity flow}
\label{sec:ns2d_cavity}
The lid-driven cavity flow is a popular benchmark for steady incompressible Navier--Stokes equations with nonlinear convection and coupled velocity--pressure fields. Following the setup in~\cite{hao2024pinnacle}, we consider the system on the spatial domain $\Omega_s=[0,1]^2$:
\begin{equation}
\begin{aligned}
    (\mathbf{u} \cdot \nabla)\mathbf{u} + \nabla p - \frac{1}{\mathrm{Re}}\Delta \mathbf{u} &= 0, & x &\in \Omega_s,\\
    \nabla \cdot \mathbf{u} &= 0, & x &\in \Omega_s,
\end{aligned}
\end{equation}
where $\mathbf{u} = (u_1, u_2)$ denotes the velocity field, $p$ denotes the pressure, and $\mathrm{Re} = 100$ is the Reynolds number. 
Denoting the top boundary by $\Gamma_1$ and the union of the 
left, right, and bottom boundaries by $\Gamma_2$, the flow is 
subject to
\begin{equation}
\begin{aligned}
    \mathbf{u}(x) &= \bigl(4x(1-x), 0\bigr), & x &\in \Gamma_1,\\
    \mathbf{u}(x) &= (0, 0), & x &\in \Gamma_2,\\
    p(0, 0) &= 0. 
\end{aligned}
\end{equation}

All MLP-based methods use a fully connected network with $3$ hidden layers, each with $100$ neurons.
We use $N_f = 8{,}196$ collocation points and
$N_b = 256$ boundary points. Training is performed using Adam for $30{,}000$ iterations, starting from an initial learning rate of $1 \times 10^{-3}$, 
decayed by a factor of $0.99$ every $100$ iterations.

The warm-up phase of FALM-PINN lasts for $N_{\mathrm{warm}} = 5{,}000$ iterations via Adam at a fixed learning rate of $1 \times 10^{-3}$. The Fourier feature mapping uses $D = 800$ features 
with bandwidth $\sigma = 1$. During the alternating optimization phase 
($N_{\mathrm{alter}} = 25{,}000$), the upper-level parameters $\omega$ continue to be optimized using Adam with an initial learning rate of $1 \times 10^{-4}$ decayed by a factor of $0.99$ every $100$ iterations down to a floor of 
$1 \times 10^{-5}$. Each lower-level solve 
performs $J = 2$ LM iterations. To accommodate the strong 
nonlinearity brought by the convection term, the LM damping is 
initialized at $\gamma_0 = 1 \times 10^{-5}$ and decayed by a 
factor of $0.5$ every $2{,}000$ iterations, until reaching $1\times10^{-6}$.

The remaining baseline configurations are as follows. 
RBA applies residual-based 
attention reweighting with $\eta_{\mathrm{RBA}} = 0.001$ and $\gamma_{\mathrm{RBA}} = 0.999$. 
CompleX-PINN adopts a single Cauchy layer of width $500$. 
PIKAN uses a Kolmogorov--Arnold 
network of structure $[2, 5, 5, 5, 3]$ with grid size $20$ and 
spline order $3$. SIREN employs sinusoidal activations with $\omega_0 = 30$ as prescribed in the original paper.

The relative $L^2$ error histories for all methods are presented in Fig.~\ref{fig:ns2d_convergence}. During the warm-up phase, the error reduction rate of FALM-PINN remains on the same scale as the baselines. Upon transition to the alternating optimization phase, the error decreases by nearly one order of magnitude and continues to decay smoothly, which is consistent with the two-stage behavior observed in the 2D Klein--Gordon experiment. Table~\ref{tab:ns_err} further reports the component-wise relative $L^2$ errors for $u_1$, $u_2$, and $p$. FALM-PINN achieves the lowest error in all three components and reduces the total error by about $45\%$ relative to SIREN, the most accurate baseline in this experiment.

\begin{table}[ht]
\centering
\caption{Relative $L^2$ error (mean $\pm$ standard 
deviation) for each solution component ($u_1$, $u_2$, $p$) on the 2D lid-driven cavity flow. The lowest error in each column is shown in bold.}
\label{tab:ns_err}
\setlength{\tabcolsep}{10pt}
\renewcommand{\arraystretch}{1.15}
\resizebox{\linewidth}{!}{%
\begin{tabular}{l cccc}
\toprule
Method & $u_1$ & $u_2$ & $p$ & Total\\
\midrule
Vanilla PINN      & $ 2.54\times10^{-2}\pm5.63\times10^{-3} $ & $ 4.81\times10^{-2}\pm1.21\times10^{-2} $ & $ 8.02\times10^{-2}\pm2.43\times10^{-2} $ &  $ 4.48\times10^{-2}\pm7.75\times10^{-3} $\\
RBA           & $ 1.72\times10^{-2}\pm1.53\times10^{-3} $ &$ 2.55\times10^{-2}\pm1.83\times10^{-3} $ & $ 3.69\times10^{-2}\pm4.93\times10^{-3} $ & $ 2.23\times10^{-2}\pm1.89\times10^{-3} $ \\
compleX-PINN  & $ 1.28\times10^{-2}\pm6.94\times10^{-4} $ & $ 2.22\times10^{-2}\pm1.71\times10^{-3} $  & $ 3.71\times10^{-2}\pm2.44\times10^{-3} $  & $ 1.96\times10^{-2}\pm1.16\times10^{-3} $ \\
PIKAN         & $ 4.99\times10^{-3}\pm8.37\times10^{-4} $&$ 7.72\times10^{-3}\pm1.08\times10^{-3} $ & $ 1.57\times10^{-2}\pm2.53\times10^{-3} $ & $ 7.82\times10^{-3}\pm1.43\times10^{-3} $ \\
SIREN         & $ 3.11\times10^{-3}\pm5.49\times10^{-4} $ & $ 4.28\times10^{-3}\pm1.33\times10^{-3} $ & $ 1.31\times10^{-2}\pm1.30\times10^{-3} $ &  $5.07\times10^{-3}\pm4.59\times10^{-4}$\\
\midrule
\textbf{FALM-PINN} & $ \mathbf{1.11\times10^{-3}\pm3.38\times10^{-4} }$ & $ \mathbf{2.06\times10^{-3}\pm1.87\times10^{-4}}$ &$ \mathbf{8.04\times10^{-3}\pm6.41\times10^{-4}}$ & $ \mathbf{2.78\times10^{-3}\pm1.29\times10^{-4}}$\\
\bottomrule
\end{tabular}
}%
\end{table}

\begin{figure}[htbp]
    \centering
    \includegraphics[width=0.99\linewidth]{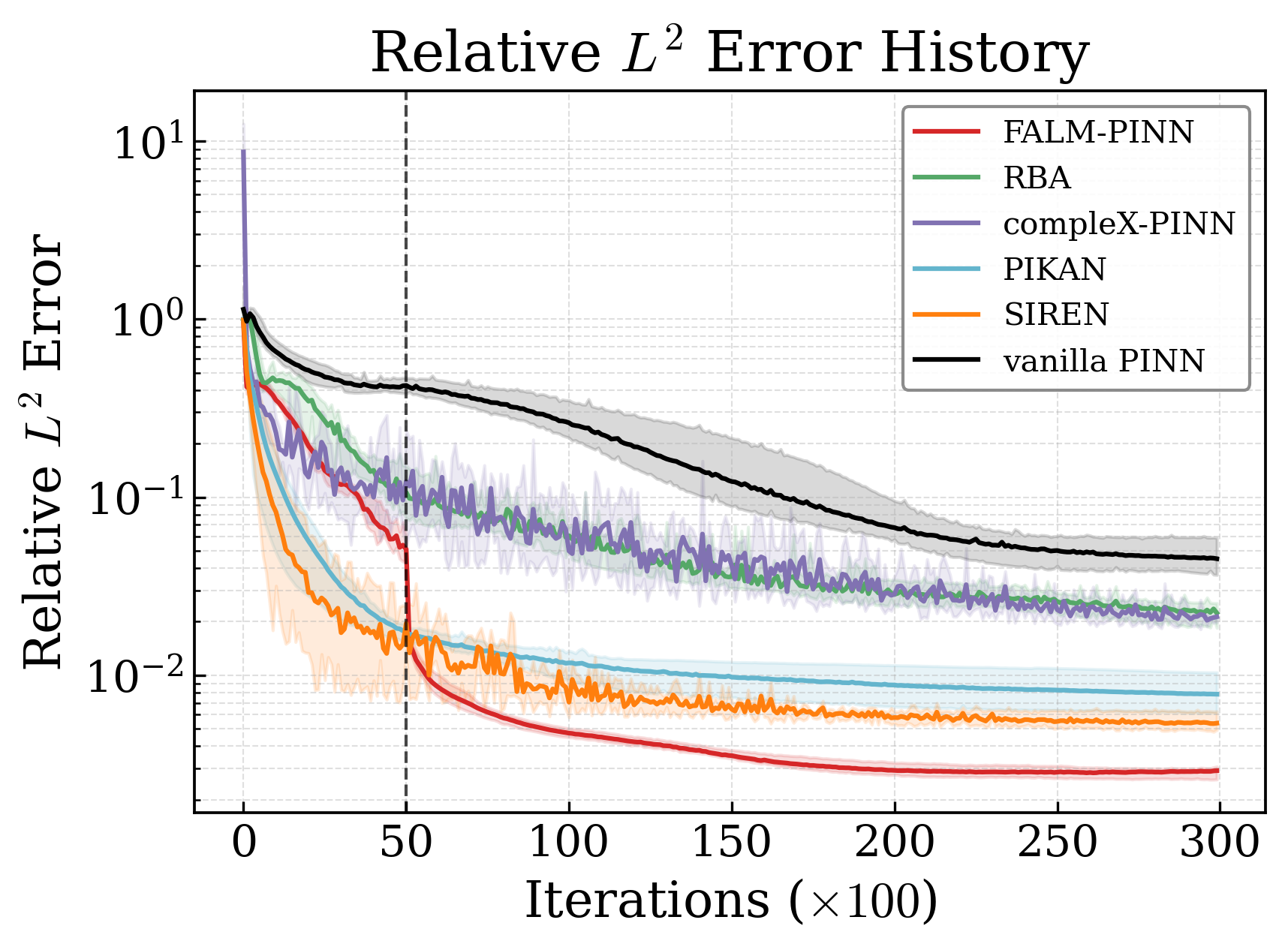}
    \caption{Relative $L^2$ error trajectories on the 2D lid-driven cavity flow. Curves denote the mean across five independent trials; shaded bands indicate the min--max range. The vertical dashed line marks the transition from the warm-up phase ($N_{\mathrm{warm}}=5{,}000$) to the alternating optimization phase.}
    \label{fig:ns2d_convergence}
\end{figure}
We further examine the point-wise absolute errors of the predicted velocity components $u_1$, $u_2$, and pressure $p$ in Fig.~\ref{fig:ns2d_fields}, computed against the reference solution shown in Fig.~\ref{fig:ns2d}. FALM-PINN achieves smaller errors across all three fields, indicating that its improvement is reflected not only in the relative $L^2$ error but also in the spatial error distribution. These results show that FALM-PINN accurately reproduces the coupled flow structure and remains effective for approximating nonlinear multi-component PDE systems.

\begin{figure}[htbp]
    \centering
    \includegraphics[width=0.99\linewidth, height=0.85\textheight, keepaspectratio]{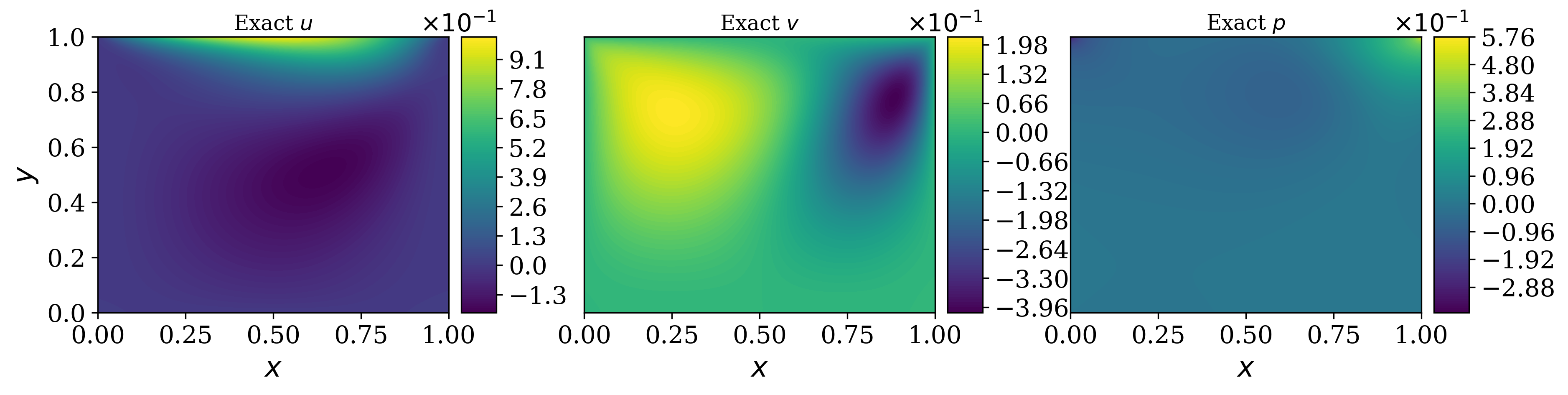}
    \caption{Reference solution of the 2D lid-driven cavity flow.}
    \label{fig:ns2d}
\end{figure}

\begin{figure}[p]
    \centering
    \includegraphics[width=0.99\linewidth, height=0.95\textheight, keepaspectratio]{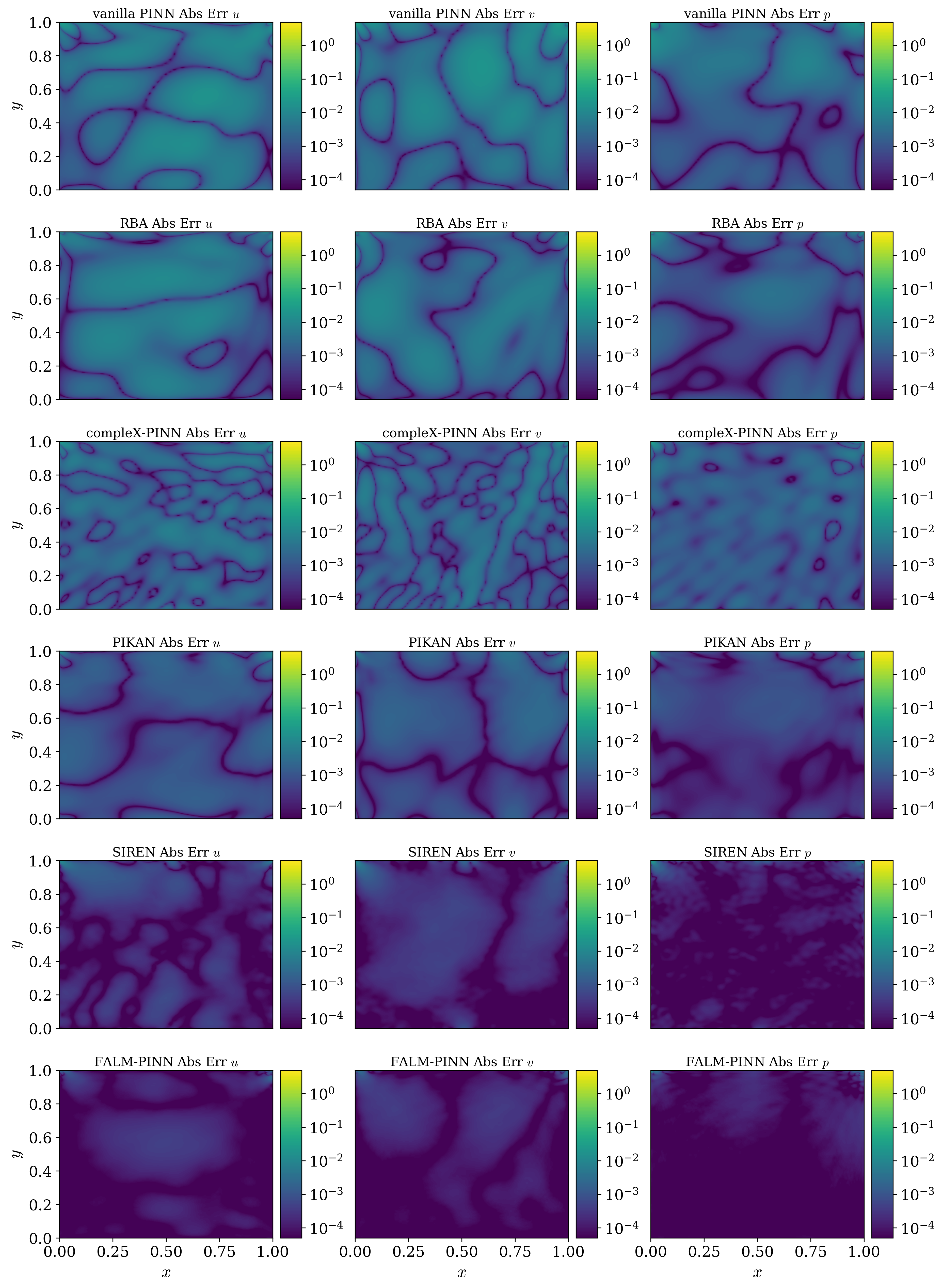}
    \caption{Point-wise absolute error maps (log scale) on the 2D lid-driven cavity flow. Columns: horizontal velocity $u_1$, vertical velocity $u_2$, and pressure $p$. Rows: vanilla PINN, RBA, compleX-PINN, PIKAN, SIREN, and FALM-PINN.}
    \label{fig:ns2d_fields}
\end{figure}

\subsection{1D viscous Burgers equation}
\label{sec:burgers1d}
The viscous Burgers equation is a nonlinear convection--diffusion benchmark that develops shock-like steep-gradient structures under small viscosity. Following the formulation in~\cite{raissi2019physics}, we consider the equation on the spatio-temporal domain 
$\Omega = T\times \Omega_s $ with $\Omega_s = [-1, 1]$ and 
$T = [0, 1]$:
\begin{equation}
\begin{aligned}
    &u_t + u u_x - \frac{0.01}{\pi} u_{xx} = 0,  \quad (t,x) \in \Omega,\\
    &u(0, x) = -\sin(\pi x),\quad  x \in \Omega_s,\\
    &u(t, -1) = u(t, 1) = 0,\quad  t \in T.
\end{aligned}
\end{equation}
All MLP-based methods share the same network architecture of $8$ hidden layers of width $20$. We use $N_f = 10{,}000$ collocation points and $N_i + N_b = 100$ combined initial and boundary points. The baselines are optimized by Adam with an initial learning rate of $5 \times 10^{-3}$, exponentially decayed by a factor of $0.7$ every $1{,}000$ iterations down to a floor of $1 \times 10^{-5}$, and trained for $40{,}000$ iterations for all methods to fully converge.

During the warm-up phase of FALM-PINN ($N_{\mathrm{warm}} = 10{,}000$), all parameters are jointly trained by Adam at a fixed learning rate of $1 \times 10^{-3}$. The Fourier feature mapping uses $D = 800$ features and bandwidth $\sigma = 1$. The subsequent alternating optimization phase ($N_{\mathrm{alter}} = 30{,}000$) then updates the upper-level parameters $\omega$, starting from a learning rate of $1 \times 10^{-4}$ and following the same exponential decay scheduler as the baselines. Each lower-level solve performs $J = 1$ LM iteration with the damping initialized at $\gamma = 1\times10^{-5}$ and decayed by a factor of $0.5$ for every $1000$ iterations, until below $1\times10^{-6}$.

For the baselines, RBA~\cite{anagnostopoulos2024residual} uses residual-based 
attention reweighting with $\eta_{\mathrm{RBA}} = 0.001$ and 
$\gamma_{\mathrm{RBA}} = 0.999$. 
CompleX-PINN~\cite{si2025complexphysicsinformedneuralnetwork} 
adopts a single Cauchy layer of width $200$. 
PIKAN~\cite{wang2025kolmogorov} uses a Kolmogorov--Arnold network 
of structure $[2, 5, 5, 5, 1]$ with grid size $15$ and spline 
order $3$. SIREN~\cite{sitzmann2020implicit} is used with its default configuration with $\omega_0 = 30$.

Table~\ref{tab:burgers1d} and Fig.~\ref{fig:burgers1d_convergence} summarize the final results and the convergence behavior of all methods. During the warm-up phase, FALM-PINN follows a decay trend comparable to the baselines. After switching to the alternating optimization phase, however, the error undergoes a sharp decline and reaches a final relative $L^2$ error of $5.10\times10^{-5}$, while all baseline methods remain above $1\times10^{-3}$. This two-stage pattern echoes the behavior observed in the previous experiments and further supports the effectiveness of the framework. However, the training time and memory cost of FALM-PINN are higher than those of the baselines on this problem. This is because FALM-PINN relies on the residual matrix, whose dimension grows with the number of collocation points. Consequently, its computational and memory costs increase more rapidly than the baselines as the dataset size grows.

\begin{table}[htbp]
\centering
\caption{1D viscous Burgers equation: relative $L^2$ errors (mean $\pm$ std over five independent runs), GPU time per $1000$ iterations, and peak GPU memory for all methods. The lowest value in each column is shown in bold.}
\label{tab:burgers1d}
\setlength{\tabcolsep}{6pt}
\renewcommand{\arraystretch}{1.15}
\begin{tabular}{lccc}
\toprule
Method & Relative $L^2$ error & Time (s) & Memory (GB) \\
\midrule
Vanilla PINN  & $2.16\times10^{-2}\pm5.99\times10^{-3}$ & $\mathbf{3.29}$& $\mathbf{0.08}$\\
RBA           & $3.03\times10^{-3}\pm9.16\times10^{-4}$ & $3.95$ & $0.09$ \\
compleX-PINN  & $3.53\times10^{-2}\pm4.47\times10^{-3}$ & $4.03$ & $0.44$\\
PIKAN         & $2.13\times10^{-2}\pm7.69\times10^{-3}$ & $15.59$ & $0.43$\\
SIREN         & $1.53\times10^{-2}\pm5.65\times10^{-3}$ & $4.12$ &  $0.10$\\
\midrule
\textbf{FALM-PINN} & $\mathbf{5.10\times10^{-5}\pm1.22\times10^{-5}}$ &$14.35$ & $1.25$ \\
\bottomrule
\end{tabular}
\end{table}

\begin{figure}
    \centering
    \includegraphics[width=0.99\linewidth]{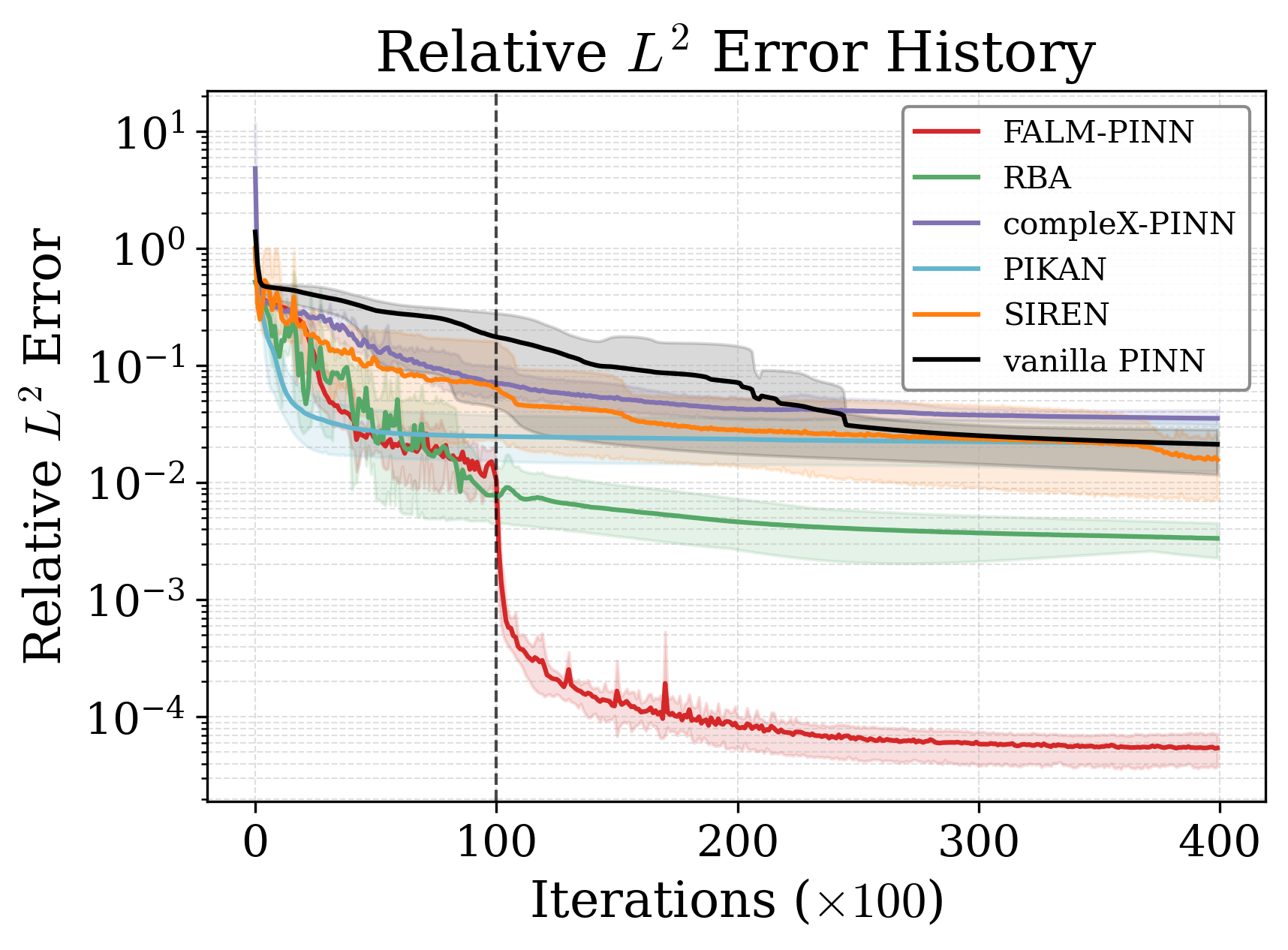}
    \caption{Relative $L^2$ error trajectories on the 1D viscous Burgers equation. Curves denote the mean across five independent trials, and shaded bands indicate the min--max range.}
    \label{fig:burgers1d_convergence}
\end{figure}

The reference solution and the corresponding point-wise absolute errors are presented in Figs.~\ref{fig:burgers1d} and~\ref{fig:burgers1d_heatmap}. FALM-PINN maintains lower errors across the domain, including the steep-gradient region that develops near $x=0$. These results indicate that the improvement is not limited to the global relative $L^2$ error but also appears in the localized region where the nonlinear convection term makes the solution difficult to approximate.

\begin{figure}[htbp]
    \centering
    \includegraphics[width=0.5\linewidth]{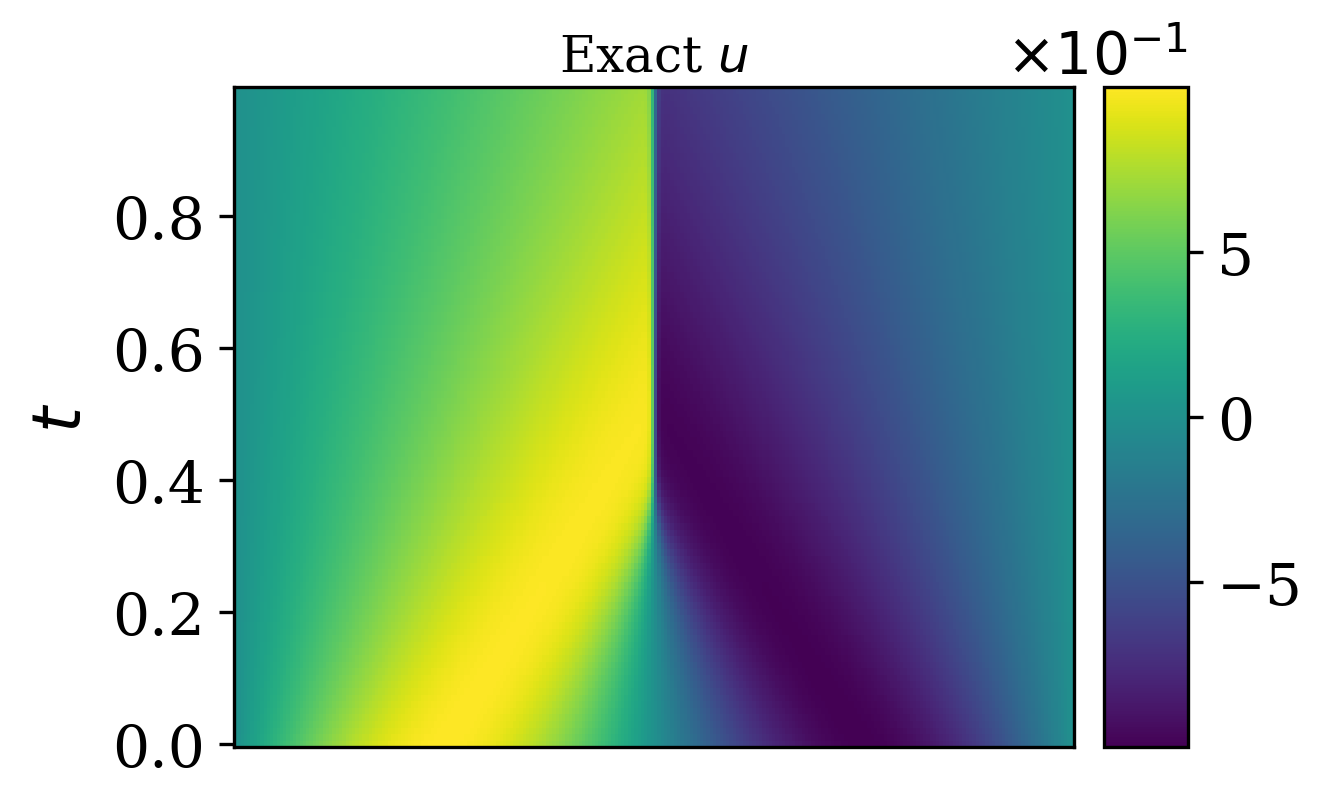}
    \caption{Reference solution of the 1D viscous Burgers equation.}
    \label{fig:burgers1d}
\end{figure}

\begin{figure}[htbp]
    \centering
    \includegraphics[width=0.99\linewidth, height=0.95\textheight, keepaspectratio]{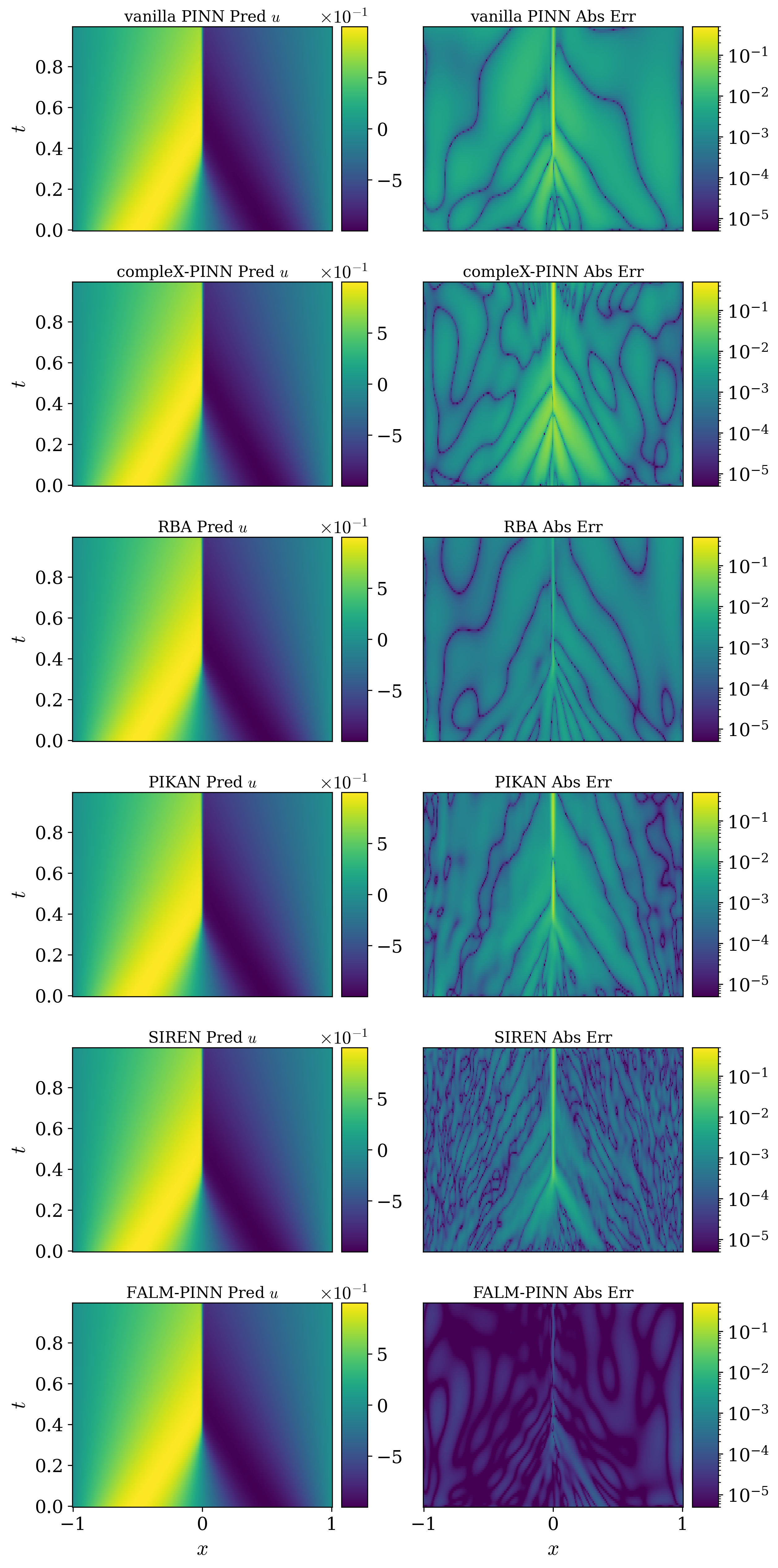}
    \caption{Predictions (left column) and error maps in log scale (right column) on the 1D viscous Burgers equation for vanilla PINN, RBA, compleX-PINN, PIKAN, SIREN and FALM-PINN respectively.}
    \label{fig:burgers1d_heatmap}
\end{figure}

To better understand the behavior of all predictions near the shock, we present temporal prediction slices for all methods at fixed spatial locations $x \in \{-0.1, 0, 0.1\}$ against the reference solution in Fig.~\ref{fig:burgers1d_slices}. FALM-PINN closely tracks the reference at all three locations throughout the temporal domain. In contrast, all baselines reproduce only the overall profile and exhibit visible deviations in the critical region near $t = 0.5$. These slices further validate the robustness of FALM-PINN in resolving nonlinear PDEs with shock structure.

\begin{figure}[htbp]
    \centering
   \includegraphics[width=0.99\linewidth]{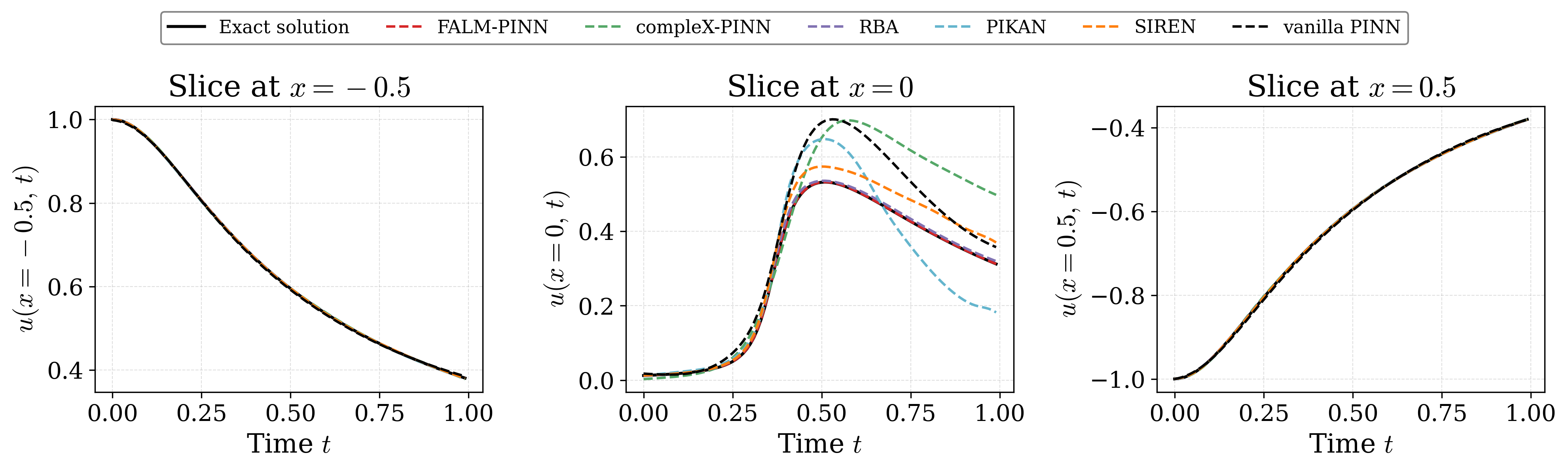}
   \caption{Prediction slices over time at fixed spatial positions $x \in \{-0.1, 0, 0.1\}$ on the 1D viscous Burgers equation, comparing all methods against the exact solution.}
    \label{fig:burgers1d_slices}
\end{figure}

\subsection{Comparison with IFeF-PINN}
\label{sec:compare_ifef}
Recall that IFeF-PINN is designed primarily for scalar, linear PDEs, where the lower-level problem is convex and admits a closed-form solution~\cite{wu2025iterative}. For nonlinear PDEs, this problem becomes nonconvex and no longer has a closed-form solution; IFeF-PINN therefore updates the coefficients to a local minimizer every $N_{\text{lower}}$ epochs using L-BFGS. As a generic optimizer, it captures the overall curvature of the loss but ignores the least-squares structure of the residual. To assess the effectiveness of the lower-level LM solver of FALM-PINN, we compare the two methods on the three nonlinear benchmarks (2D Klein–Gordon, 1D Korteweg--de Vries, and 1D viscous Burgers), keeping the network architecture, the Fourier-feature mapping, the training set, and all training hyperparameters identical.

Table~\ref{tab:compare_ifef} reports the final relative $L^2$ errors. FALM-PINN is one to two orders of magnitude more accurate than IFeF-PINN on every nonlinear benchmark. The L-BFGS lower-level of IFeF-PINN does reduce the loss, but it converges slowly and stalls at a substantially higher error, because it does not exploit the least-squares geometry that the LM update captures through the Gauss--Newton Hessian and adaptive damping. These results show that, for the nonlinear decoupled problem, the structure-exploiting LM solve is markedly more accurate than a generic gradient-based optimizer, even though both share the same Fourier-enhanced basis.

\begin{table}[htbp]
  \centering
  \caption{Comparison of IFeF-PINN and FALM-PINN on three scalar nonlinear PDE benchmarks. All settings other than the lower-level solver are kept identical. Entries report the final relative $L^2$ error (mean $\pm$ std over five independent runs).}
  \label{tab:compare_ifef}
  \setlength{\tabcolsep}{5pt}
  \renewcommand{\arraystretch}{1.1}
  \resizebox{\linewidth}{!}{%
  \begin{tabular}{lcc}
  \toprule
  Benchmark & IFeF-PINN & FALM-PINN \\
  \midrule
  2D Klein--Gordon    & $1.28\times10^{-3} \pm 1.87\times10^{-4}$ & $\mathbf{3.50\times10^{-5} \pm 4.88\times10^{-6}}$ \\
  1D Korteweg--de Vries   & $1.32\times10^{-2} \pm 4.28\times10^{-3}$ & $\mathbf{4.42\times10^{-4} \pm 1.81\times10^{-4}}$ \\
  1D viscous Burgers  & $2.46\times10^{-3} \pm 6.71\times10^{-4}$ & $\mathbf{5.10\times10^{-5} \pm 1.22\times10^{-5}}$ \\
  \bottomrule
  \end{tabular}%
  } %
\end{table}

\section{Conclusion and future work}
\label{sec:conclusion}
In this work, we proposed FALM-PINN, an alternating training framework
that decouples Fourier-enhanced basis learning from projection-coefficient
fitting. The upper-level problem learns an adaptive latent basis, while the
lower-level problem exploits the nonlinear least-squares structure of the
physics residual through damped Levenberg--Marquardt updates. We establish global convergence for this formulation, overcoming the limitation of conventional PINNs. Empirically, FALM-PINN significantly improves accuracy compared to state-of-the-art methods across a variety of high-frequency, nonlinear, and coupled benchmarks, achieving substantially lower
relative $L^2$ errors than the considered baselines.

A promising direction for future work is to develop adaptive strategies for selecting the Fourier bandwidth, thereby constructing a more compact basis without sacrificing approximation accuracy.
Combining FALM-PINN with adaptive resampling methods may further reduce the computational cost and memory consumption of the method. These
developments would improve the scalability of the framework to more complex
and high-dimensional PDEs. In parallel, integrating FALM-PINN with advanced
neural architectures or operator-learning frameworks offers further
opportunities to enhance its expressiveness and generalization capability.
Such extensions may lead to more accurate and reliable physics-informed
solvers and help narrow the performance gap between learning-based approaches
and classical numerical methods.

\appendix
\section{Worked example: scalar conservation laws}
\label{app:example}
To illustrate the assembly of the residual vector and its Jacobian in the
lower-level solver, we instantiate the framework of
Section~\ref{sec:lower_level} on the scalar conservation law:
\begin{equation}
\label{equ:conservation_law}
u_t + \bigl[P(u)\bigr]_x = 0,
\quad (t,x)\in(0,T]\times(0,1),
\end{equation}
subject to Dirichlet boundary conditions $u(t,0)=u(t,1) = 0$ and initial condition $u(0,x)=u_0(x)$, where $P:\mathbb{R}\to\mathbb{R}$ denotes a smooth nonlinear flux function. This scalar form covers a wide range of one-dimensional nonlinear conservation laws. Representative examples include $P(u)=\tfrac12 u^2$, which leads to the inviscid Burgers' equation, and $P(u) = u e^{-u}$, commonly used in traffic flow models.

For the scalar conservation law~\eqref{equ:conservation_law} with $m=M=1$, the coefficient vector reduces to $\beta\in\mathbb{R}^{2D}$
and the approximate solution at the $j$-th step is $u_j(x) = \psi_D(x)^\top\beta^{(j)}$.
By the chain rule $[P(u)]_x = P'(u)u_x$, the nonlinear PDE operator
at a collocation point $x_f^i$ gives:
\begin{equation}
\begin{aligned}
\label{equ:example_pde_eval}
\mathscr{N}[u_{\omega,\beta}](x_f^i)
&= (\partial_t\psi_D^i)^\top\beta^{(j)} + 
      P'(u_j^i)
      (\partial_x\psi_D^i)^\top\beta^{(j)},\\
\bm{M}_f(\beta) &= \big[ \mathscr{N}[u_{\omega,\beta}](x_f^1), \dots, \mathscr{N}[u_{\omega,\beta}](x_f^{N_f}) \big]^\top,
\end{aligned}
\end{equation}
where $\bm{M}_f(\beta)\in \mathbb{R}^{N_f}$ and we abbreviate $\psi_D^i = \psi_D(x_f^i)$, $u_j^i =
{\psi_D^i}^\top\beta^{(j)}$, and $\partial_t\psi_D^i$,
$\partial_x\psi_D^i$ denotes the partial
derivatives of the feature map. Differentiating $\mathscr{N}[u_{\omega,\beta}](x_f^i)$ with respect to $\beta$ yields the Jacobian matrix $\nabla_\beta \bm{M}_f(\beta^{(j)})\in \mathbb{R}^{N_f \times 2D}$:
\begin{equation}
\label{equ:example_jacobian}
\nabla_\beta \bm{M}_f(\beta^{(j)})
= \begin{bmatrix}
(\partial_t\psi_D^1)^\top
+ P'(u_j^1)(\partial_x\psi_D^1)^\top
+ P''(u_j^1)(\partial_x\psi_D^1)^\top\beta^{(j)}(\psi_D^1)^\top \\
\vdots \\
(\partial_t\psi_D^{N_f})^\top
+ P'(u_j^{N_f})(\partial_x\psi_D^{N_f})^\top
+ P''(u_j^{N_f})(\partial_x\psi_D^{N_f})^\top\beta^{(j)} (\psi_D^{N_f})^\top 
\end{bmatrix}.
\end{equation}
For the inviscid Burgers' flux $P(u)=\tfrac{1}{2}u^2$, we have $P'(u_j^i) = u_j^i$ and $P''(u_j^i) = 1$. Consequently, each row of the physics Jacobian $\nabla_\beta\bm{M}_f(\beta^{(j)})$ simplifies to:
\begin{equation}
\label{equ:burgers_jacobian_row}
    (\partial_t\psi_D^i)^\top + u_j^i(\partial_x\psi_D^i)^\top + \bigl((\partial_x\psi_D^i)^\top\beta^{(j)}\bigr)(\psi_D^i)^\top \in\mathbb{R}^{1\times 2D}.
\end{equation}
For the traffic flow model $P(u) = u e^{-u}$, we have $P'(u_j^i) = (1 - u_j^i)e^{-u_j^i}$ and $P''(u_j^i) = (u_j^i - 2)e^{-u_j^i}$. Accordingly, each row of $\nabla_\beta\bm{M}_f(\beta^{(j)})$ becomes:
\begin{equation}
\label{equ:traffic_jacobian_row}
(\partial_t\psi_D^i)^\top
+ (1 - u_j^i)e^{-u_j^i}(\partial_x\psi_D^i)^\top
+ (u_j^i - 2)e^{-u_j^i}\bigl((\partial_x\psi_D^i)^\top\beta^{(j)}\bigr)(\psi_D^i)^\top
\in\mathbb{R}^{1\times 2D}.
\end{equation}
For a linear flux $P(u)=au$, $P'(u_j^i)=a$ and $P''(u_j^i)=0$, so each row
becomes $(\partial_t\psi_D^i)^\top + a(\partial_x\psi_D^i)^\top$,
which is independent of~$\beta^{(j)}$.

With the Jacobian assembled, the update $\beta^{(j+1)} = \beta^{(j)} + \Delta\beta^{(j)}$ is obtained by solving the normal equation~\eqref{equ:lm_normal_eq}, where the Gram matrix expands as:
\begin{equation}
\label{equ:example_gram}
\mathbf{J}_j^\top\mathbf{J}_j
= \frac{1}{N_b}\bm{M}_b^\top\bm{M}_b
  + \frac{\lambda}{N_f}
    \bigl(\nabla_\beta\bm{M}_f(\beta^{(j)})\bigr)^\top
    \nabla_\beta\bm{M}_f(\beta^{(j)})
  \in\mathbb{R}^{2D\times 2D}.
\end{equation}
The boundary Gram matrix $\bm{M}_b^\top\bm{M}_b$ is constant, while the physics Gram matrix is recomputed at each iteration.

\bibliographystyle{elsarticle-num} 
\bibliography{reference}

\begin{thebibliography}{10}
\expandafter\ifx\csname url\endcsname\relax
  \def\url#1{\texttt{#1}}\fi
\expandafter\ifx\csname urlprefix\endcsname\relax\def\urlprefix{URL }\fi
\expandafter\ifx\csname href\endcsname\relax
  \def\href#1#2{#2} \def\path#1{#1}\fi

\bibitem{leveque2007finite}
R.~J. LeVeque, Finite difference methods for ordinary and partial differential equations: steady-state and time-dependent problems, SIAM, 2007.

\bibitem{zienkiewicz1977finite}
O.~C. Zienkiewicz, R.~L. Taylor, P.~Nithiarasu, J.~Zhu, The finite element method, Vol.~3, Elsevier, 1977.

\bibitem{ciarlet2002finite}
P.~G. Ciarlet, The finite element method for elliptic problems, SIAM, 2002.

\bibitem{brenner2008mathematical}
S.~C. Brenner, L.~R. Scott, The mathematical theory of finite element methods, Springer, 2008.

\bibitem{boyd2001chebyshev}
J.~P. Boyd, Chebyshev and Fourier spectral methods, Courier Corporation, 2001.

\bibitem{canuto2006spectral}
C.~Canuto, M.~Y. Hussaini, A.~Quarteroni, T.~A. Zang, Spectral methods, Vol. 285, Springer, 2006.

\bibitem{shen2011spectral}
J.~Shen, T.~Tang, L.-L. Wang, Spectral methods: algorithms, analysis and applications, Vol.~41, Springer Science \& Business Media, 2011.

\bibitem{weinan2003heterognous}
E.~Weinan, B.~Engquist, The heterognous multiscale methods, Communications in Mathematical Sciences 1~(1) (2003) 87--132.

\bibitem{abdulle2012heterogeneous}
A.~Abdulle, E.~Weinan, B.~Engquist, E.~Vanden-Eijnden, The heterogeneous multiscale method, Acta Numerica 21 (2012) 1--87.

\bibitem{efendiev2009multiscale}
Y.~Efendiev, T.~Y. Hou, Multiscale finite element methods: theory and applications, Springer Science \& Business Media, 2009.

\bibitem{iserles2005efficient}
A.~Iserles, S.~P. N{\o}rsett, Efficient quadrature of highly oscillatory integrals using derivatives, Proceedings of the Royal Society A: Mathematical, Physical and Engineering Sciences 461~(2057) (2005) 1383--1399.

\bibitem{raissi2019physics}
M.~Raissi, P.~Perdikaris, G.~E. Karniadakis, Physics-informed neural networks: A deep learning framework for solving forward and inverse problems involving nonlinear partial differential equations, Journal of Computational Physics 378 (2019) 686--707.

\bibitem{karniadakis2021physics}
G.~E. Karniadakis, I.~G. Kevrekidis, L.~Lu, P.~Perdikaris, S.~Wang, L.~Yang, Physics-informed machine learning, Nature Reviews Physics (2021).

\bibitem{hu2024tackling}
Z.~Hu, K.~Shukla, G.~E. Karniadakis, K.~Kawaguchi, Tackling the curse of dimensionality with physics-informed neural networks, Neural Networks 176 (2024) 106369.

\bibitem{costabal2024delta}
F.~S. Costabal, S.~Pezzuto, P.~Perdikaris, $\delta$-{PINNs}: Physics-informed neural networks on complex geometries, Engineering Applications of Artificial Intelligence 127 (2024) 107324.

\bibitem{raissi2020hidden}
M.~Raissi, A.~Yazdani, G.~E. Karniadakis, Hidden fluid mechanics: Learning velocity and pressure fields from flow visualizations, Science 367~(6481) (2020) 1026--1030.

\bibitem{jin2021nsfnets}
X.~Jin, S.~Cai, H.~Li, G.~E. Karniadakis, {NSFnets} ({Navier-Stokes} flow nets): Physics-informed neural networks for the incompressible {Navier-Stokes} equations, Journal of Computational Physics 426 (2021) 109951.

\bibitem{chen2020physics}
Y.~Chen, L.~Lu, G.~E. Karniadakis, L.~Dal~Negro, Physics-informed neural networks for inverse problems in nano-optics and metamaterials, Optics express 28~(8) (2020) 11618--11633.

\bibitem{yang2021b}
L.~Yang, X.~Meng, G.~E. Karniadakis, B-pinns: Bayesian physics-informed neural networks for forward and inverse pde problems with noisy data, Journal of Computational Physics 425 (2021) 109913.

\bibitem{bastek2025physics}
J.-H. Bastek, W.~Sun, D.~Kochmann, Physics-informed diffusion models, in: International Conference on Learning Representations, Vol. 2025, 2025, pp. 3360--3385.

\bibitem{wang2025source}
Z.~Wang, A.~Harting, M.~Barreau, M.~M. Zavlanos, K.~H. Johansson, Source-guided flow matching, arXiv preprint arXiv:2508.14807 (2025).

\bibitem{shu2023physics}
D.~Shu, Z.~Li, A.~B. Farimani, A physics-informed diffusion model for high-fidelity flow field reconstruction, Journal of Computational Physics 478 (2023) 111972.

\bibitem{li2021fourier}
Z.~Li, N.~B. Kovachki, K.~Azizzadenesheli, B.~Liu, K.~Bhattacharya, A.~Stuart, A.~Anandkumar, Fourier neural operator for parametric partial differential equations, in: International Conference on Learning Representations, 2021.

\bibitem{lu2021learning}
L.~Lu, P.~Jin, G.~Pang, Z.~Zhang, G.~E. Karniadakis, Learning nonlinear operators via {DeepONet} based on the universal approximation theorem of operators, Nature Machine Intelligence 3~(3) (2021) 218--229.

\bibitem{rahaman2019spectral}
N.~Rahaman, A.~Baratin, D.~Arpit, F.~Draxler, M.~Lin, F.~Hamprecht, Y.~Bengio, A.~Courville, On the spectral bias of neural networks, in: International Conference on Machine Learning, PMLR, 2019, pp. 5301--5310.

\bibitem{xu2025understanding}
Z.-Q.~J. Xu, L.~Zhang, W.~Cai, On understanding and overcoming spectral biases of deep neural network learning methods for solving {PDEs}, Journal of Computational Physics (2025) 113905.

\bibitem{krishnapriyan2021characterizing}
A.~Krishnapriyan, A.~Gholami, S.~Zhe, R.~Kirby, M.~W. Mahoney, Characterizing possible failure modes in physics-informed neural networks, Advances in Neural Information Processing Systems 34 (2021) 26548--26560.

\bibitem{wang2024respecting}
S.~Wang, S.~Sankaran, P.~Perdikaris, Respecting causality for training physics-informed neural networks, Computer Methods in Applied Mechanics and Engineering 421 (2024) 116813.

\bibitem{wang2022and}
S.~Wang, X.~Yu, P.~Perdikaris, When and why {PINNs} fail to train: A neural tangent kernel perspective, Journal of Computational Physics 449 (2022) 110768.

\bibitem{song2024loss}
Y.~Song, H.~Wang, H.~Yang, M.~L. Taccari, X.~Chen, Loss-attentional physics-informed neural networks, Journal of Computational Physics 501 (2024) 112781.

\bibitem{anagnostopoulos2024residual}
S.~J. Anagnostopoulos, J.~D. Toscano, N.~Stergiopulos, G.~E. Karniadakis, Residual-based attention in physics-informed neural networks, Computer Methods in Applied Mechanics and Engineering 421 (2024) 116805.

\bibitem{si2026convolution}
C.~Si, M.~Yan, Convolution-weighting method for the physics-informed neural network: A primal-dual optimization perspective, Journal of Computational Physics 555 (2026) 113911.

\bibitem{zhao2026casual}
C.~Zhao, X.~Xie, W.~Chen, Casual attention: Adaptive enforcement of causality in physics-informed neural networks, Journal of Computational Physics (2026) 115071.

\bibitem{wu2023comprehensive}
C.~Wu, M.~Zhu, Q.~Tan, Y.~Kartha, L.~Lu, A comprehensive study of non-adaptive and residual-based adaptive sampling for physics-informed neural networks, Computer Methods in Applied Mechanics and Engineering 403 (2023) 115671.

\bibitem{gao2023active}
W.~Gao, C.~Wang, Active learning based sampling for high-dimensional nonlinear partial differential equations, Journal of Computational Physics 475 (2023) 111848.

\bibitem{lau2024pinnacle}
G.~K.~R. Lau, A.~Hemachandra, S.-K. Ng, B.~K.~H. Low, {PINNACLE}: {PINN} adaptive collocation and experimental points selection, in: The Twelfth International Conference on Learning Representations, 2024.

\bibitem{tancik2020fourier}
M.~Tancik, P.~Srinivasan, B.~Mildenhall, S.~Fridovich-Keil, N.~Raghavan, U.~Singhal, R.~Ramamoorthi, J.~Barron, R.~Ng, Fourier features let networks learn high frequency functions in low dimensional domains, Advances in Neural Information Processing Systems 33 (2020) 7537--7547.

\bibitem{wang2021eigenvector}
S.~Wang, H.~Wang, P.~Perdikaris, On the eigenvector bias of {Fourier} feature networks: From regression to solving multi-scale {PDEs} with physics-informed neural networks, Computer Methods in Applied Mechanics and Engineering 384 (2021) 113938.

\bibitem{sitzmann2020implicit}
V.~Sitzmann, J.~Martel, A.~Bergman, D.~Lindell, G.~Wetzstein, Implicit neural representations with periodic activation functions, Advances in neural information processing systems 33 (2020) 7462--7473.

\bibitem{si2025complexphysicsinformedneuralnetwork}
C.~Si, M.~Yan, X.~Li, Z.~Xia, \href{https://arxiv.org/abs/2502.04917}{Complex physics-informed neural network} (2025).
\newblock \href {http://arxiv.org/abs/2502.04917} {\path{arXiv:2502.04917}}.
\newline\urlprefix\url{https://arxiv.org/abs/2502.04917}

\bibitem{zeng2026nurbs}
X.~Zeng, Y.~Zhu, A nurbs-based parameterization physics-informed neural network with an adaptive architecture for solving pdes, Journal of Computational Physics 562 (2026) 114980.

\bibitem{ZhaoEtAl24}
Z.~Zhao, X.~Ding, B.~A. Prakash, {PINNsFormer}: A transformer-based framework for physics-informed neural networks, in: The Twelfth International Conference on Learning Representations, 2024.

\bibitem{wang2025kolmogorov}
Y.~Wang, J.~Sun, J.~Bai, C.~Anitescu, M.~S. Eshaghi, X.~Zhuang, T.~Rabczuk, Y.~Liu, Kolmogorov--arnold-informed neural network: A physics-informed deep learning framework for solving forward and inverse problems based on kolmogorov--arnold networks, Computer Methods in Applied Mechanics and Engineering 433 (2025) 117518.

\bibitem{jagtap2020extended}
A.~D. Jagtap, G.~E. Karniadakis, Extended physics-informed neural networks (xpinns): A generalized space-time domain decomposition based deep learning framework for nonlinear partial differential equations, Communications in Computational Physics 28~(5) (2020).

\bibitem{wang2024piratenets}
S.~Wang, B.~Li, Y.~Chen, P.~Perdikaris, Piratenets: Physics-informed deep learning with residual adaptive networks, Journal of Machine Learning Research 25~(402) (2024) 1--51.

\bibitem{mustajab2024physics}
A.~H. Mustajab, H.~Lyu, Z.~Rizvi, F.~Wuttke, Physics-informed neural networks for high-frequency and multi-scale problems using transfer learning, Applied Sciences 14~(8) (2024) 3204.

\bibitem{wang2024multi}
Y.~Wang, C.-Y. Lai, Multi-stage neural networks: Function approximator of machine precision, Journal of Computational Physics 504 (2024) 112865.

\bibitem{cyr2020robust}
E.~C. Cyr, M.~A. Gulian, R.~G. Patel, M.~Perego, N.~A. Trask, Robust training and initialization of deep neural networks: An adaptive basis viewpoint, in: Mathematical and Scientific Machine Learning, PMLR, 2020, pp. 512--536.

\bibitem{huang2006extreme}
G.-B. Huang, Q.-Y. Zhu, C.-K. Siew, Extreme learning machine: theory and applications, Neurocomputing 70~(1-3) (2006) 489--501.

\bibitem{dwivedi2020physics}
V.~Dwivedi, B.~Srinivasan, Physics informed extreme learning machine (pielm)--a rapid method for the numerical solution of partial differential equations, Neurocomputing 391 (2020) 96--118.

\bibitem{dong2021local}
S.~Dong, Z.~Li, Local extreme learning machines and domain decomposition for solving linear and nonlinear partial differential equations, Computer Methods in Applied Mechanics and Engineering 387 (2021) 114129.

\bibitem{wu2025iterative}
Y.~Wu, M.~Aguiar, K.~H. Johansson, M.~Barreau, Iterative training of physics-informed neural networks with fourier-enhanced features, arXiv preprint arXiv:2510.19399 (2025).

\bibitem{levenberg1944method}
K.~Levenberg, A method for the solution of certain non-linear problems in least squares, Quarterly of applied mathematics 2~(2) (1944) 164--168.

\bibitem{marquardt1963algorithm}
D.~W. Marquardt, An algorithm for least-squares estimation of nonlinear parameters, Journal of the society for Industrial and Applied Mathematics 11~(2) (1963) 431--441.

\bibitem{wilson2016deep}
A.~G. Wilson, Z.~Hu, R.~Salakhutdinov, E.~P. Xing, Deep kernel learning, in: Artificial Intelligence and Statistics, PMLR, 2016, pp. 370--378.

\bibitem{rahimi2007random}
A.~Rahimi, B.~Recht, Random features for large-scale kernel machines, Advances in Neural Information Processing Systems 20 (2007).

\bibitem{kingma2014adam}
D.~P. Kingma, J.~Ba, Adam: A method for stochastic optimization, arXiv preprint arXiv:1412.6980 (2014).

\bibitem{coleman2013calculus}
R.~Coleman, Calculus on normed vector spaces, Choice Reviews Online (2013).

\bibitem{penwarden2023unified}
M.~Penwarden, A.~D. Jagtap, S.~Zhe, G.~E. Karniadakis, R.~M. Kirby, A unified scalable framework for causal sweeping strategies for physics-informed neural networks ({PINN}s) and their temporal decompositions, Journal of Computational Physics 493 (2023) 112464.
\newblock \href {https://doi.org/10.1016/j.jcp.2023.112464} {\path{doi:10.1016/j.jcp.2023.112464}}.

\bibitem{shannon2006communication}
C.~E. Shannon, Communication in the presence of noise, Proceedings of the IRE 37~(1) (2006) 10--21.

\bibitem{hao2024pinnacle}
Z.~Hao, J.~Yao, C.~Su, H.~Su, Z.~Wang, F.~Lu, Z.~Xia, Y.~Zhang, S.~Liu, L.~Lu, et~al., Pinnacle: A comprehensive benchmark of physics-informed neural networks for solving pdes, Advances in Neural Information Processing Systems 37 (2024) 76721--76774.

\end{thebibliography}

\end{document}